\documentclass[11pt]{article}

\usepackage[final]{acl}

\usepackage{times}
\usepackage{latexsym}

\usepackage[T1]{fontenc}

\usepackage[utf8]{inputenc}

\usepackage{microtype}

\usepackage{inconsolata}

\usepackage{graphicx}

\usepackage{bm}
\newcommand{\vx}{\bm{x}}
\newcommand{\vy}{\bm{y}}

\newcommand{\g}{\bm{g}}
\newcommand{\f}{\bm{f}}

\newcommand{\m}{\texttt{m}}

\newcommand{\X}{\mathcal{X}}
\newcommand{\Y}{\mathcal{Y}}

\newcommand{\E}{\mathbb{E}}
\newcommand{\N}{\mathbb{N}}
\newcommand{\R}{\mathbb{R}}
\newcommand{\D}{\mathbb{D}}

\usepackage[]{mdframed}
\usepackage{enumitem}
\usepackage{tikz}
\usepackage{wrapfig}
\usepackage{tabularx}
\usepackage{mdframed}
\usepackage{amsmath}
\usepackage{amssymb}
\usepackage{mathtools}
\usepackage{amsthm}
\usepackage{booktabs}
\theoremstyle{plain}
\newtheorem{theorem}{Theorem}[section]
\newtheorem{proposition}[theorem]{Proposition}
\newtheorem{lemma}[theorem]{Lemma}

\theoremstyle{definition}

\theoremstyle{remark}

\usepackage{fontawesome5}

\title{MDM-Prime-v2: Binary Encoding and Index Shuffling Enable \\ Scaling of Diffusion Language Models}

\author{
 \textbf{Chen-Hao Chao\textsuperscript{1}},
 \textbf{Wei-Fang Sun\textsuperscript{2}},
 \textbf{Junwei Quan\textsuperscript{1}},
 \textbf{Chun-Yi Lee\textsuperscript{3}},
 \textbf{Rahul G. Krishnan\textsuperscript{1}}
\\
 \textsuperscript{1}University of Toronto, Vector Institute
\\
 \textsuperscript{2}NVIDIA AI Technology Center
\\
 \textsuperscript{3}National Taiwan University
\\
 \small{
   \textbf{Correspondence:} \href{rahulgk@cs.toronto.edu}{chchao@cs.toronto.edu}, \href{rahulgk@cs.toronto.edu}{rahulgk@cs.toronto.edu}
 }
 \\
 \small{
   \faGithub~\textbf{Code:} \url{https://github.com/chen-hao-chao/mdm-prime-v2}
 }
}

\begin{document}
\maketitle
\begin{abstract}
Masked diffusion models (MDM) exhibit superior generalization when learned using a \underline{P}a\underline{r}t\underline{i}al \underline{m}asking schem\underline{e} (Prime). This approach converts tokens into sub-tokens and models the diffusion process at the sub-token level. We identify two limitations of the MDM-Prime framework. First, we find that the functional form of the subtokenizer significantly increases the cross-entropy loss in the objective when paired with commonly used Byte-Pair-Encoding (BPE) tokenizers. Second, we lack tools to guide the hyperparameter choice of the token granularity in the subtokenizer.  To address these limitations, we analyze the optimal design of the subtokenizer that minimizes MDM-Prime training objective and develop MDM-Prime-v2, a masked diffusion language model which incorporates \textit{Binary Encoding} and \textit{Index Shuffling}. Our analysis characterizes how token granularity and sub-token entropy influence the training objective and downstream performance, providing principled criteria for subtokenizer design. When extending the model size to 1.1B parameters, MDM-Prime-v2 demonstrates superior average zero-shot accuracy across eight commonsense reasoning benchmarks, outperforming similar-sized baselines including GPT-Neo, OPT, Pythia, Bloom, SMDM, and TinyLLaMA.
\end{abstract}

\section{Introduction}
\begin{figure}[t!]
    \centering
    \vspace{-0.5em}
    \includegraphics[width=\linewidth]{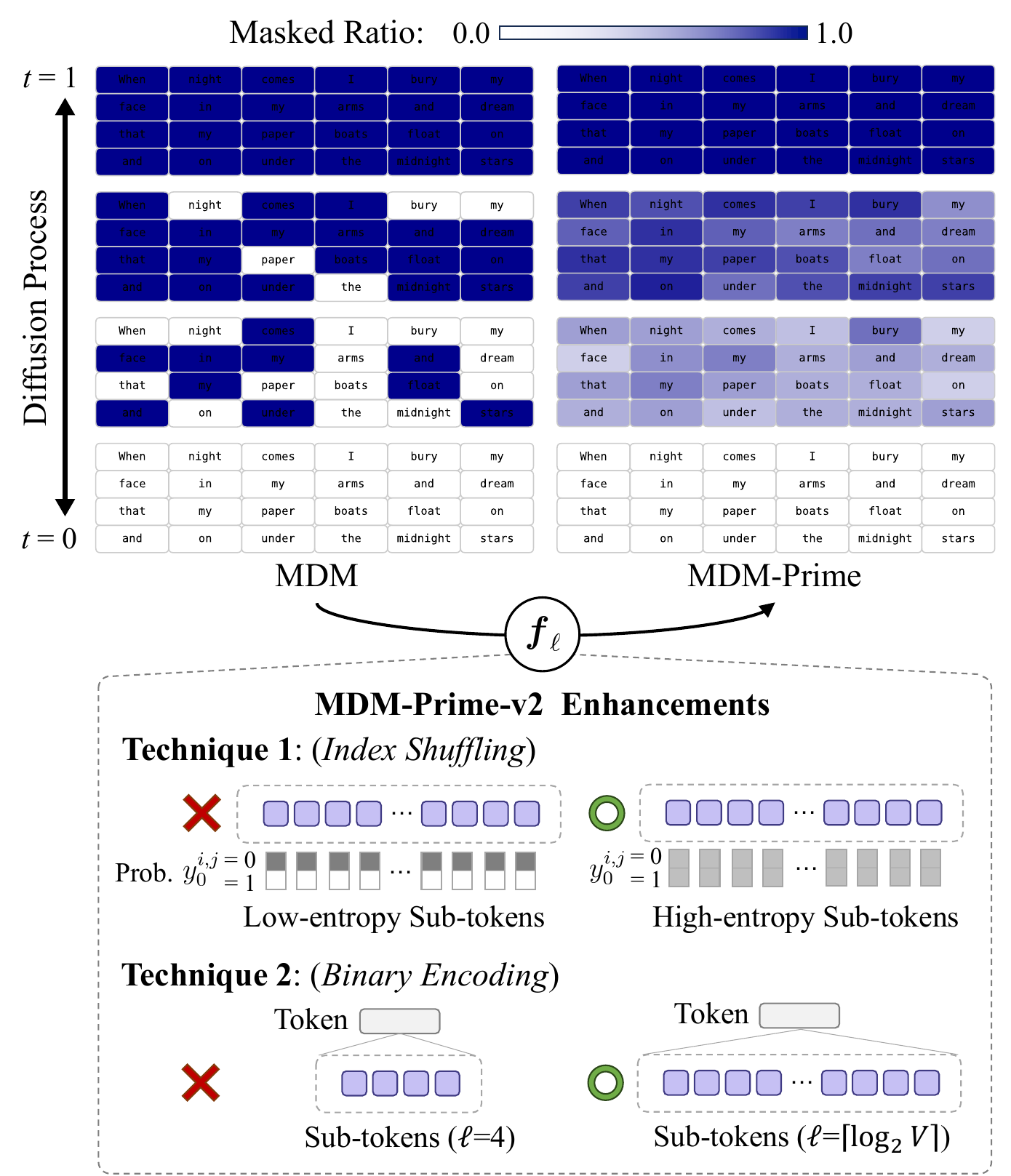}
    \caption{Overview of MDM, MDM-Prime, and the MDM-Prime-v2 enhancements. $\ell$ is the token granularity in base-$b$ encoding and $V$ is the vocabulary size. $y_0^{i,j}$ is a sub-token with positional indices $i$ and $j$. MDM-Prime-v2 adopts two techniques: Technique 1 employs an index shuffling operation to encode high-entropy sub-tokens. Technique 2 sets $\ell=\lceil\log_2 V\rceil$ for $\f_\ell$, which performs binary encoding.  The gray cell intensity in the lower section represents sub-token probability.}
    \vspace{-1.5em}
    \label{fig:teaser}
\end{figure}

Generative pretraining serves as the cornerstone of scalable language modeling~\cite{kaplan2020scalinglawsneurallanguage}. Autoregressive models (ARM) (e.g., \citet{radford2019language}) and masked diffusion models (MDM) (e.g., \citet{sahoo2024simplifieddiff, nie2025llmdiff}) represent two paradigms: the former utilizes the chain rule to decompose the joint likelihood of the data, while the latter employs stochastic unmasking to reconstruct data from masked sequences. 

\citet{chao2025mdmprime} demonstrated that subtokenization improves the expressivity of MDM latent representations, enabling partial word editing during diffusion steps and substantially improving sampling behavior. This framework was formalized as MDM-Prime, a generalized class of MDMs. In MDMs, each token takes one of two states, \textit{masked} or \textit{unmasked}. MDM-Prime introduces a richer set of intermediate latent representations via partially masked tokens. MDM-Prime uses a subtokenizer $\f_\ell$ with adjustable token granularity $\ell$ to encode each token into a sequence of $\ell$ sub-tokens. By modeling the diffusion process at the sub-token level, partially masked tokens can be produced naturally as intermediate transition states, leading to a fine-grained denoising process, as illustrated in the upper section of Fig.~\ref{fig:teaser}. 

We investigate two fundamental questions regarding $\f_\ell$. First, we identify a link between the subtokenizer, which determines how the model perceives inputs, and the tokenizer, which defines the data distribution. Our analysis of $\f_\ell$ reveals that high-entropy sub-tokens decrease the conditional (predictive) entropy, effectively minimizing the objective. Second, instead of treating $\ell$ as an empirically tuned hyperparameter, we characterize the relationship between $\ell$ and downstream performance, and establish a principled criterion for its selection. Based on this insight, we introduce a technique called index shuffling to permute token indices. We integrate these improvements into a unified framework named MDM-Prime-v2.

We perform a scaling analysis of MDM-Prime-v2. By varying compute budgets from $3\times10^{18}$ to $3\times10^{20}$ FLOPs, we determine compute-optimal configurations and derive its scaling behavior via Chinchilla analysis~\cite{hoffmann2022chinchilla}. We scale MDM-Prime-v2 to 1.1B parameters and our model demonstrates superior performance over similar-sized baselines including GPT-Neo~\cite{gpt-neo}, OPT~\cite{zhang2022optopenpretrainedtransformer}, Pythia~\cite{biderman2023pythia}, Bloom~\cite{workshop2023bloom176bparameteropenaccessmultilingual}, SMDM~\cite{nie2025scalingmaskeddiffusionmodels} and TinyLLaMA~\cite{zhang2024tinyllama} on zero-shot commonsense reasoning benchmarks.
The contributions of this work are summarized as follows:
\begin{itemize}[leftmargin=10pt]
\vspace{-0.7em}
    \item We analyze the MDM-Prime objective and characterize how sub-token entropy and token granularity $\ell$ influence its optimal value, providing a theoretical basis for subtokenizer design.\vspace{-0.3em}
    \item Based on the aforementioned analysis, we establish two practical techniques: a technique to enhance sub-token entropy, called index shuffling, and a selection criterion for $\ell$ under which $\f_\ell$ performs binary encoding. \vspace{-0.3em}
    \item We demonstrate that MDM-Prime-v2 outperforms MDMs and ARMs at the 1.1B parameter scale, achieving the highest average zero-shot accuracy across eight benchmarks. 
\end{itemize}

\section{Background}
\label{sec:background}


\subsection{Masked Diffusion Models}
\label{sec:background:mdm}
Let $\vx = [x^1, \cdots, x^L] \in \X^L$ be a sequence of $L$ tokens 
, where $x^i$ denotes an element of $\vx$ and $\X \triangleq \{0,\cdots,V-1\}$ represents a set of $V$ token indices. Given a continuous time variable $t \in [0,1]$, $\vx_0\in\mathcal{X}^L$ denote the sample drawn from the data distribution $q_{\text{data}}$, and $\vx_t\in (\X \cup \{\m\} )^L \triangleq \tilde{\X}^L$ denotes the latent variable introduced by the forward diffusion process, where $\texttt{m}$ represents the masked token. Let $\delta_{x'}(x)$ be the Kronecker delta function, which equals $1$ if $x = x'$ and $0$ otherwise. The forward diffusion process is performed through a kernel $q_{\alpha}(\vx_{t}|\vx_0)=\prod_{i=1}^L q_{\alpha}(x_{t}^i| x^i_0)$ following a strictly decreasing time-dependent scheduling function $\alpha_t: [0,1] \to [0,1]$. The element-wise kernel $q_{\alpha}$ is defined as follows~\cite{sahoo2024simplifieddiff}:
\begin{equation}
\label{eq:kernel}
 q_{\alpha}(x^i_t\mid x_0^i)=(1-\alpha_t)\delta_{\texttt{m}}(x^i_t)+\alpha_t \delta_{x_0^i}(x^i_t).
\end{equation}
Let $s,t\in[0,1]$ be time variables such that $s < t$. The reverse diffusion process is performed by iteratively applying $p(\vx_{s}|\vx_t)=\E_{p(\vx_0|\vx_t)}[q_\alpha(\vx_{s}|\vx_{t},\vx_{0})]$~\cite{austin2021structurediff}, by first drawing $\vx_0\sim p(\cdot\mid \vx_t)$ and then sampling $\vx_{s} \sim q_\alpha(\cdot \mid\vx_{t},\vx_{0})$ in each timestep, starting from $t=1$. The reverse kernel $q_\alpha(\vx_{s}|\vx_{t},\vx_{0})=\prod_{i=1}^L q_\alpha(x^i_{s}|x^i_{t},x^i_{0})$ is defined as:
\vspace{-0.2em}
\begin{equation}
\label{eq:mdm_reverse_kernel}
\begin{aligned}
    q_\alpha&(x^i_{s}|x^i_{t},x^i_{0}) \\
    &=\begin{cases}
      \delta_{x^i_t}(x^i_{s}),&\hspace{-0.6em}\text{if } x^i_t \in \X,\\
      \frac{1-\alpha_{s}}{1-\alpha_t}\delta_{\texttt{m}}(x^i_{s})+\frac{\alpha_{s}-\alpha_{t}}{1-\alpha_t} \delta_{x^i_0}(x^i_{s}),&\hspace{-0.6em}\text{if }x^i_t = \texttt{m}.
    \end{cases}    
\end{aligned}
\end{equation}
Intuitively, each masked token transitions to its original value $x^i_{0}$ with probability $\frac{\alpha_{s}-\alpha_{t}}{1-\alpha_t}$ and retains the masked value with probability $\frac{1-\alpha_{s}}{1-\alpha_t}$.

To learn $p(\vx_0|\vx_t)$, one common approach is to minimize a weighted cross-entropy loss:
\begin{equation}
\label{eq:mdm_elbo}
\begin{aligned}
    \mathcal{L}=\int_0^1 w(t) \E_{q_\alpha(\vx_0,\vx_t)}\left[-\log p(\vx_0|\vx_t) \right] dt,
\end{aligned}
\end{equation}
where $q(\vx_0,\vx_t)\triangleq q_\text{data}(\vx_0)q_{\alpha}(\vx_t| \vx_0)$. When $w(t)=1$, Eq.~(\ref{eq:mdm_elbo}) represents an integral of conditional cross-entropy loss~\cite{austin2021structurediff,chang2022maskgit,gat2024dfm}. Under the choice $w(t)=\frac{\alpha'_t}{\alpha_t-1}$ (where $\alpha_t'=\frac{d}{dt}\alpha_t$) and the factorized posterior $p(\vx_0|\vx_t)=\prod_{i=1}^L p (x_0^i|\vx_t)$, this objective yields a variational upper bound on the negative log-likelihood of the data distribution~\cite{sahoo2024simplifieddiff, shi2024simplifieddiff}. 


\vspace{-0.3em}
\subsection{Generalization via Partial Masking}
\label{sec:background:prime}
Given the token granularity $\ell$, MDM-Prime~\cite{chao2025mdmprime} represents each token $x^i_0\in \X$ as a sequence of $\ell$ sub-tokens $\vy^i_0=[y^{i,1}_0,\cdots,y^{i,\ell}_0]\in \Y^\ell$ via an invertible function $f_\ell:\X\to\Y^{\ell}$, known as subtokenizer, where $\Y \triangleq \{0, \cdots, b-1\}$ denotes a set of sub-token indices with $b=\lceil\sqrt[\ell]{V}\rceil\in \N$. The vectorized function $\f_\ell$ applied to the full sequence is defined as: $\f_\ell(\vx_0)\triangleq [f_\ell(x^1_0), \cdots, f_\ell(x^L_0)]=[\vy^1_0,\cdots,\vy^L_0]=[(y^{1,1}_0,\cdots,y^{1,\ell}_0),\cdots,(y^{L,1}_0,\cdots,y^{L,\ell}_0)]=\vy_0$. The maximum $\ell$ for $\f_\ell$ is $\lceil\log_2 V\rceil$, where the function encodes tokens into binary sub-tokens. This operation is composable: for granularities $\ell_1, \ell_2$ such that $\frac{\ell_2}{\ell_1} \in \mathbb{N}$, the function can be decomposed as $\f_{\ell_2} = \f_{\ell_2/\ell_1} \circ \f_{\ell_1}$. MDM-Prime uses base-$b$ encoding for $\f_\ell$, denoted as $\f_\ell\triangleq \f_\text{base-$b$}$, and implements it using a lookup table. 


The latent variable $\vy_t\in (\Y \cup \{\m\} )^{L\times \ell} \triangleq \tilde{\Y}^{L\times \ell}$ is sampled through independent masking, analogous to Eq.~(\ref{eq:kernel}). The forward kernel $q_{\alpha}(\vy_t| \vy_0)=\prod_{i=1}^{L}\prod_{j=1}^{\ell} q_{\alpha}(y^{i,j}_t| y_0^{i,j})$ is defined as follows: 
\begin{equation}
\label{eq:prime_kernel}
 q_{\alpha}(y^{i,j}_t\mid y_0^{i,j})=(1-\alpha_t)\delta_{\texttt{m}}(y^{i,j}_t)+\alpha_t \delta_{y_0^{i,j}}(y^{i,j}_t).
\end{equation}
Since $\f_\ell$ is invertible and both $\vx_0$ and $\vy_0$ are discrete, the change-of-variable principle indicates that the target distribution is invariant: $q_\text{data}(\vx_0)=q_\text{data}\circ \f^{-1}_{\ell}(\vy_0)\triangleq q_\text{data\_y}(\vy_0)$. 
Therefore, MDM-Prime approximates the same objective as MDM by substituting $\vx_0$ as $\vy_0$. The objective can be expressed in a similar way as Eq.~(\ref{eq:mdm_elbo}):
\begin{equation}
\label{eq:prime_elbo}
\begin{aligned}
\mathcal{L}^{(\ell)}=\int_0^1 w(t) \E_{q_\alpha(\vy_0,\vy_t)}\left[-\log p_{\ell} (\vy_0| \vy_t) \right] dt,
\end{aligned}
\end{equation}
where $q_\alpha(\vy_0,\vy_t)=q_\text{data\_y}(\vy_0)q_{\alpha}(\vy_t| \vy_0)$, and $p_{\ell}(\vy_0| \vy_t)$ denotes the MDM-Prime model. The weighting factor $w(t)=\frac{\alpha'_t}{\alpha_t-1}$ is the same as standard MDMs. We elaborate on how to approximate the variational bound of MDM-Prime using Eq.~(\ref{eq:prime_elbo}) in Appendix~\ref{sec:apx:proof:nll}.

Adapting the model architecture from MDM to MDM-Prime requires only a simple modification of the embedding lookup table. In this setup, sub-token embeddings are aggregated into token embeddings and subsequently processed by the neural network at the token level. This design preserves the computational cost (FLOPs) of each training iteration. A detailed description of the model architecture is provided in Appendix~\ref{sec:apx:specification:architecture} and Fig.~\ref{fig:architecture}~(a).

\section{Methodology}
This section details two proposed enhancements: Section~\ref{sec:methodology:entropy} explores improved function form for the subtokenizer and Section~\ref{sec:methodology:likelihood} discusses the token granularity selection.

\subsection{Increasing Entropy via Index Shuffling}
\label{sec:methodology:entropy}

Instead of defining $\f_\ell \triangleq \f_\text{base-$b$}$ as in MDM-Prime, this work derives the condition for $\f_\ell$ to minimize the loss function given a token granularity $\ell$. Let $\f_\text{shuffle}$ denote an \textit{Index Shuffling} operation, which permutes token indices via a lookup table (see Fig.~\ref{fig:subtokenizer}). We propose $\f_\ell \triangleq \f_\text{base-$b$} \circ \f_\text{shuffle}$, which integrates an additional $\f_\text{shuffle}$ to effectively approximate the optimum.

\begin{figure}[t]
    \centering
    \includegraphics[width=0.95\linewidth]{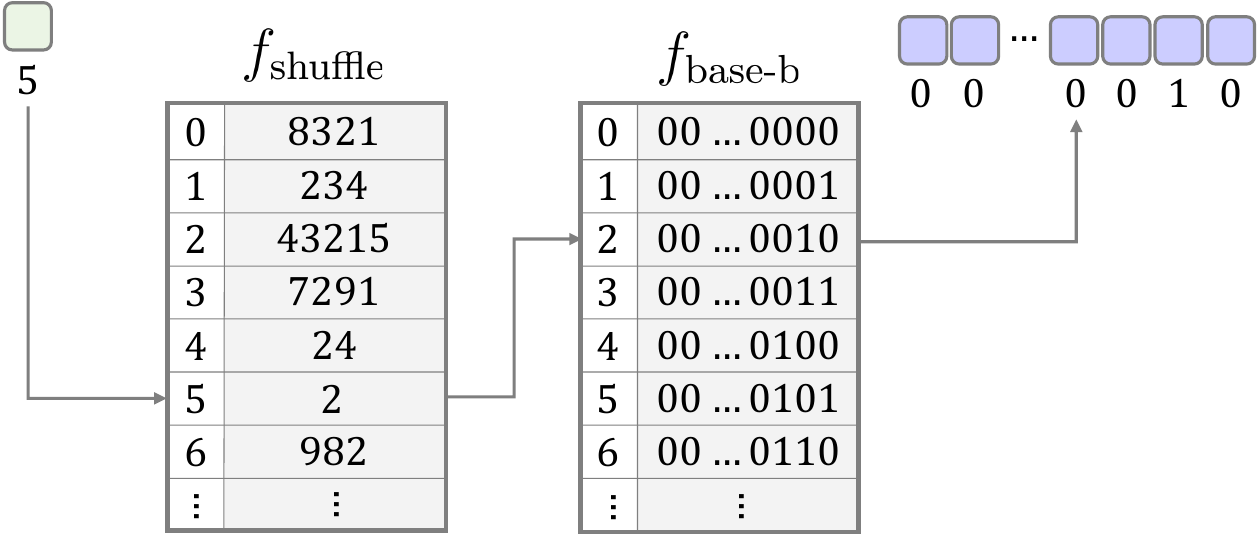}
    \vspace{-0.25em}
    \caption{Implementation of the subtokenizer $f_\ell= f_\text{base-$b$} \circ f_\text{shuffle}$. The process utilizes two lookup tables to map indices efficiently; the inverse operation $f_\ell^{-1}$ is performed by reversing the lookup sequence. `5' denotes the token ID and `00$\cdots$0010' is its sub-token representation. See Appendix~\ref{sec:apx:specification:subtokenizer} for a detailed description.}
    \label{fig:subtokenizer}
    \vspace{-1em}
\end{figure}
To illustrate this design, we first characterize the effect of $\f_\ell$ on the objective (Eq.~(\ref{eq:prime_elbo})). We isolate the term in the objective that depends on the transformation. We present this decomposition in \textbf{Proposition~\ref{proposition:decomposition}} (proof in Appendix~\ref{sec:apx:maximum_ent_subtokenizer}).

\begin{figure*}[ht!]
    \centering    \vspace{-0.7em}\includegraphics[width=\linewidth]{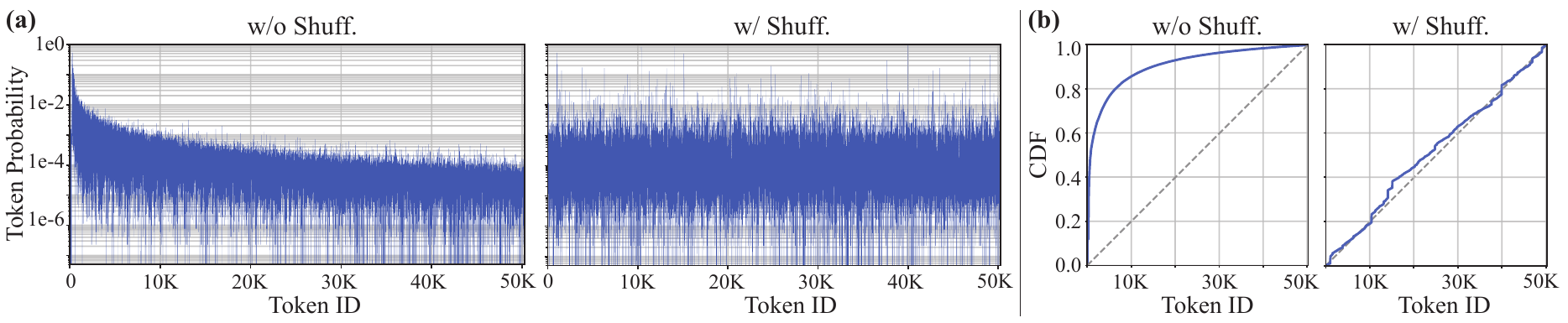}
    \vspace{-1.75em}
    \caption{(a) Token probability and (b) cumulative distribution function (CDF) over token indices of the GPT-2 tokenizer~\cite{radford2019language} evaluated on the C4 dataset~\cite{raffel2023exploringlimitstransferlearning}. The left and right subplots of (a) and (b) show setups without and with index shuffling, respectively. The gray dashed line in (b) represents the CDF of a uniform token distribution. }
    \label{fig:tok_prob}
    \vspace{-0.75em}
\end{figure*}

\begin{proposition}
\label{proposition:decomposition}
The loss can be decomposed into $\f_\ell$-independent and $\f_\ell$-dependent terms as follows:
\begin{equation}
\label{eq:dependent_independent}
\begin{aligned}
&\inf_{\f_\ell} \inf_{p_\ell} \mathcal{L}^{(\ell)} \\
&= \int_{0}^1 w(t) \Big(\!\!\!\underbrace{\mathcal{H}(\vy_0, \vy_t)}_\text{Independent on $\f_\ell$.}\!\!\!\!-\,\,\sup_{\f_\ell} \underbrace{\mathcal{H}(\vy_t)}_\text{Dependent on $\f_\ell$.} \!\!\!\!\!\!\Big)dt,
\end{aligned}
\end{equation}
where $\mathcal{H}(\vy_0, \vy_t)\triangleq \E_{q_\alpha(\vy_0,\vy_t)}[-\log q_\alpha(\vy_0,\vy_t)]$ and $\mathcal{H}(\vy_t)\triangleq \E_{q_\alpha(\vy_t)}[-\log q_\alpha(\vy_t)]$ represent the joint entropy of $\vy_0,\vy_t$ and the entropy of $\vy_t$.
\end{proposition}
The $\f_\ell$-independent term corresponds to the joint negative entropy $-\mathcal{H}(\vy_0, \vy_t) = -(\mathcal{H}(\vy_0) + \mathcal{H}(\vy_t | \vy_0))$, where $\mathcal{H}(\vy_0) = \mathcal{H}(\vx_0)$ remains constant due to the invertibility of $\f_\ell$, while $\mathcal{H}(\vy_t | \vy_0)$ is determined solely by the forward kernel in Eq.~(\ref{eq:prime_kernel}). On the other hand, the $\f_\ell$-dependent term suggests that the optimal $\f_\ell$ should maximize the entropy of $\vy_t$, which reaches its optimum when each unmasked $y_t^{i,j}$ is uniformly distributed on $\Y$, as shown in \textbf{Proposition~\ref{proposition:entropy_max}}:

\begin{proposition}
\label{proposition:entropy_max}
The entropy of $\vy_t$ is bounded:
\begin{equation}
\label{eq:yt_entropy_bound}
\mathcal{H}(\vy_t) \leq L\ell (h(\alpha_t)+\alpha_t \log b),
\end{equation}
where $h(\alpha_t)\triangleq -(1-\alpha_t) \log (1-\alpha_t) - \alpha_t \log \alpha_t$. The equality holds if and only if each unmasked $y_t^{i,j}$ is uniformly distributed on $\Y$.
\end{proposition}

\begin{figure}[t]
    \centering
    \includegraphics[width=\linewidth]{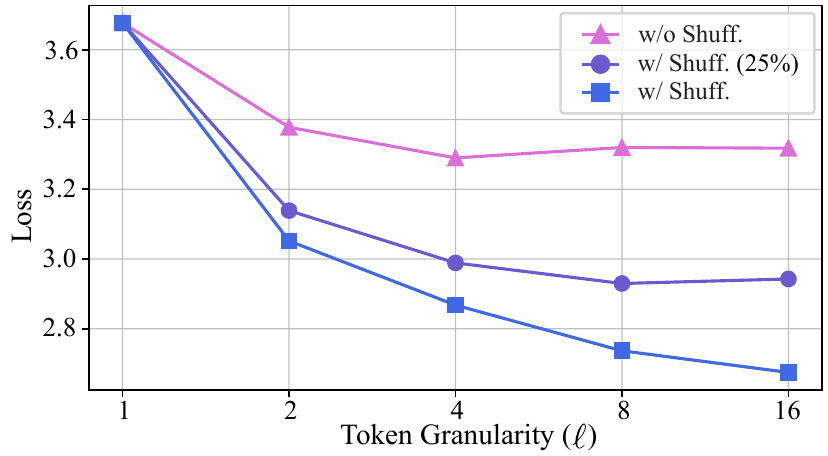}
    \vspace{-1.5em}
    \caption{Ablation study on the loss used in MDM-Prime under three setups: `w/o Shuff.,' `w/ Shuff. (25\%)', and `w/ Shuff.' $\lceil\log_2 V\rceil=16$ represents the maximum of $\ell$. All models are trained using $10^{19}$ FLOPs. The experiments are conducted on C4.}
    \label{fig:apx:target_length_scale}
\end{figure}
Although \textbf{Propositions~\ref{proposition:decomposition}} and \textbf{\ref{proposition:entropy_max}} identify high-entropy sub-tokens as the ideal case for optimality, the sub-tokens produced by directly applying base-$b$ encoding to the token indices from commonly-used BPE tokenizers exhibit low entropy. An example of the GPT-2 tokenizer~\cite{radford2019language} is presented in the `w/o Shuff.' column of Table~\ref{tab:methodology:entropy}. This occurs because BPE is constructed by iteratively merging the most frequent subword pairs, appending each newly merged pair as a new token index until the desired vocabulary size is reached. By construction, more frequent tokens receive lower indices, while less frequent tokens receive higher indices. Therefore, without deliberately shuffling the token indices, token probability is inversely proportional to the token index (see the left subplots with title `w/o Shuff' in Fig.~\ref{fig:tok_prob}~(a) and (b)). Directly encoding these token indices using base-$b$ encoding results in sub-tokens with low entropy, contradicting the maximization goal in Eqs.~(\ref{eq:dependent_independent}) and (\ref{eq:yt_entropy_bound}). To effectively disrupt the inherent token index structure, we propose the following technique:
\begin{table}[t]
    \renewcommand{\arraystretch}{1.05}
    \newcommand{\boldtoprule}{\toprule[2.5pt]}
    \newcommand{\boldbottomrule}{\bottomrule[2.5pt]}
    \centering
    \huge
    \resizebox{\linewidth}{!} {%
    \begin{tabular}{lcccc}
        \boldtoprule
          & w/o Shuff. &  w/ Shuff. (25\%) &  w/ Shuff. & Maximum \\ 
         \midrule
            $\ell = 16$                & 0.8146            & 0.9038         & 0.9936                     & 1.0000              \\
            $\ell = 8$                 & 1.6152            & 1.8040        & 1.9811                     & 2.0000              \\
            $\ell = 4$                 & 3.2051            & 3.5360        & 3.8553                     & 3.9069              \\
            $\ell = 2$                 & 6.1489            & 6.7880       & 7.2803                     & 7.8138 \\
        \boldbottomrule
    \end{tabular}}
    \vspace{-0.25em}
    \caption{Entropy of sub-tokens across different $\ell$. The results are evaluated on C4. The maximum entropy is given by $-\log_2 \frac{1}{b}$, where $b=\lceil\sqrt[\ell]{V}\rceil$ and $V=50,257$. $\lceil\log_2 V\rceil=16$ represents the maximum of $\ell$. Three configurations are compared: (1) `w/o Shuff.,' the baseline setup without index shuffling, (2) `w/ Shuff. (25\%),' the setup where the shuffling operation is applied only to the first 25\% of the token indices, and (3) `w/ Shuff.,' the setup where token indices are fully shuffled.}
    \label{tab:methodology:entropy}
\end{table}

\vspace{0.25em}
\begin{mdframed}[linewidth=0.5pt, innerleftmargin=5pt, innerrightmargin=5pt, innertopmargin=5pt, innerbottommargin=5pt]
\textbf{Technique 1.}~(Index Shuffling) Randomly shuffle the token indices before performing $\f_\text{base-$b$}$.
\end{mdframed}
\vspace{0.05em}
\begin{figure*}[t]
    \centering
    \vspace{-1em}
    \includegraphics[width=\linewidth]{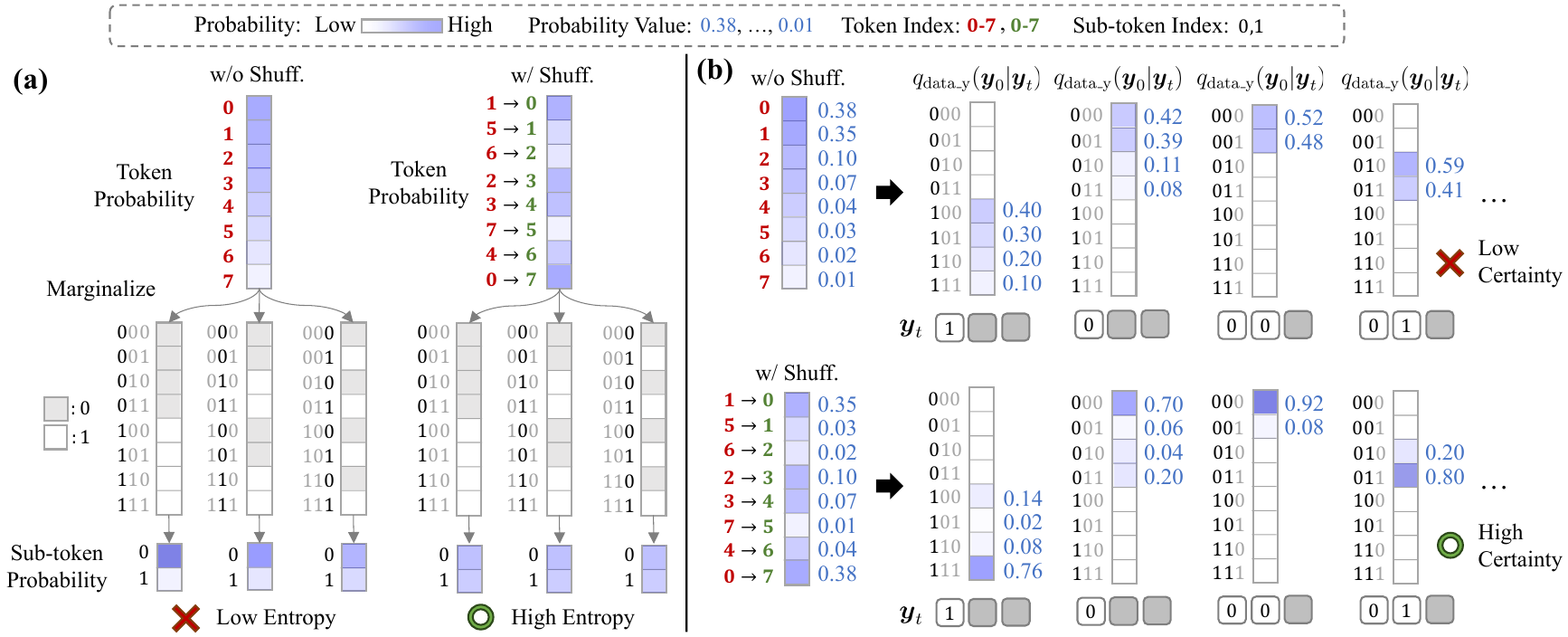}
    \vspace{-1.75em}
    \caption{Illustrative examples of the index shuffling operation proposed in Technique 1. In this plot, $V=8$, $\ell=3$ and $b=2$. Color intensity represents probability values, where darker blue denotes higher probability. (a) Marginal distributions: Marginalization is performed by summing entries corresponding to sub-token indices 0 (gray) and 1 (white), respectively. Sub-tokens encoded from standard token indices (i.e., `w/o Shuff.') exhibit low entropy, while the shuffling operation results in sub-tokens with higher entropy. A detailed numerical example is provided in Fig.~\ref{fig:apx:shuffle_example_numerical} in Appendix. (b) Conditional (predictive) distribution: The distribution $q_\text{data\_y}(\vy_0|\vy_t)$ exhibits higher certainty under index shuffling, which results in improved performance.}
    \label{fig:entropy}
    \vspace{-0.75em}
\end{figure*}

\begin{figure}[t]
    \centering
    \includegraphics[width=\linewidth]{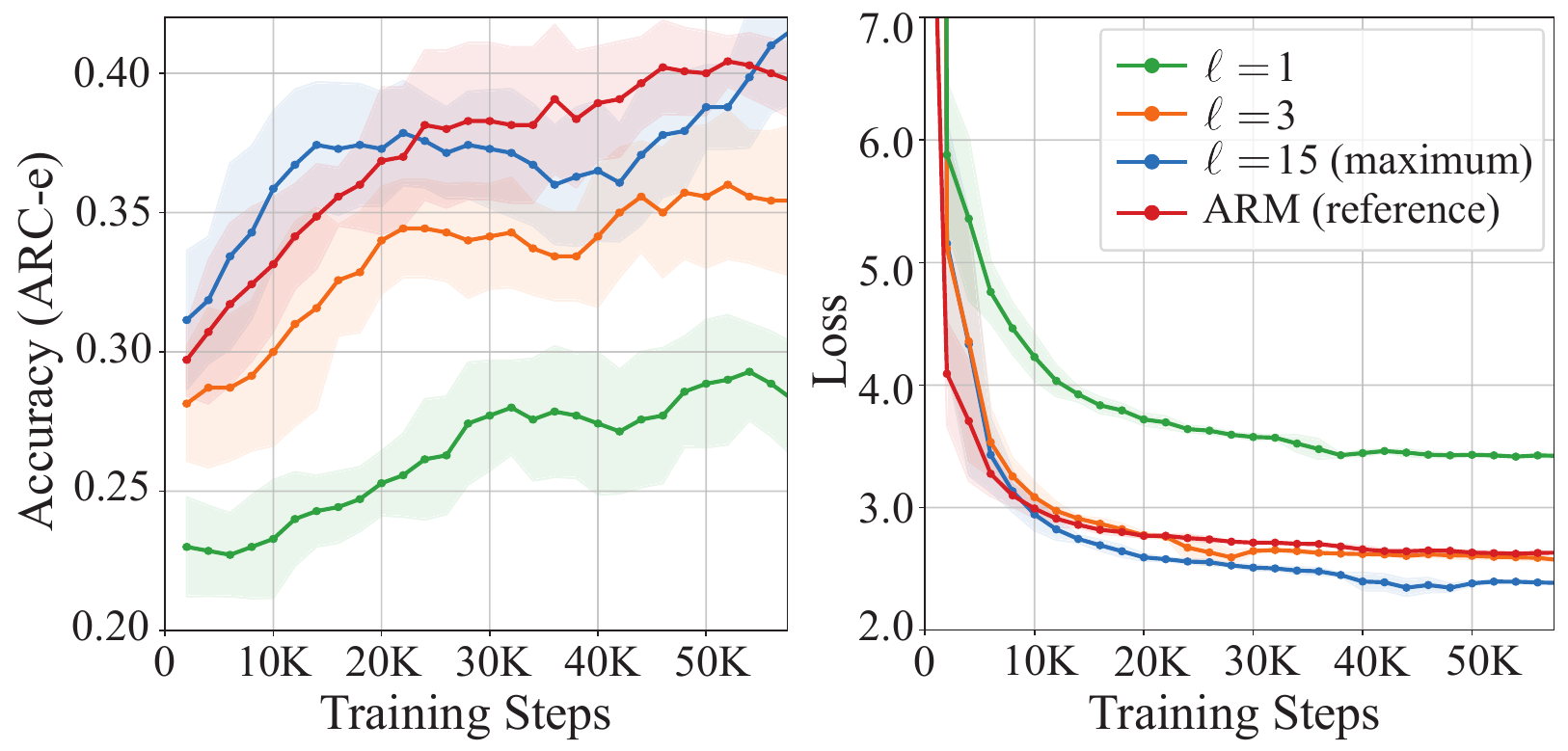}
    \vspace{-1.5em}
    \caption{Zero-shot accuracy on ARC-e of MDM-Prime under different $\ell$ selections. Accuracy and loss exhibit consistent trend. ARM is included as a reference. All models have 170M parameters and are trained on SlimPajama~\cite{cerebras2023slimpajama} tokenized using LLaMA tokenizer with $V=32,000$. The maximum $\ell$ is $\lceil\log_2 V\rceil=15$.}
    \label{fig:acc_loss}
    \vspace{-0.5em}
\end{figure}

An illustrative example is provided in Fig.~\ref{fig:entropy}~(a) to demonstrate how randomly shuffled token indices lead to higher sub-token entropy. By applying Technique 1, the average entropy approaches the theoretical maximum ($-\log_2 \frac{1}{b}$), as demonstrated in the `w/ Shuff.' and `w/ Shuff. (25\%)' columns of Table~\ref{tab:methodology:entropy}. As further illustrated in Fig.~\ref{fig:apx:target_length_scale}, the loss decreases substantially when this shuffling operation is applied, and the improvement holds across all choices of $\ell$ used in $\f_\text{base-$b$}$. Intuitively, the reduction in loss stems from increased certainty in the conditional (predictive) distribution $q_{\text{data\_y}}(\vy_0|\vy_t)$. As depicted in Fig.~\ref{fig:entropy}~(b), the index shuffling operation scatters similar probability masses across different slots. Therefore, when a specific sub-token value is observed, the conditional distribution becomes more certain, leading to a lower loss and improved sub-token prediction.

\subsection{Selection for Token Granularity}
\label{sec:methodology:likelihood}

The value of $\ell$ is critical to the performance of MDM-Prime. However, its influence on the objective and the criteria for a reliable selection have not been well explored. In this section, we propose setting $\ell$ to its maximum viable value, $\ell=\lceil\log_2 V\rceil$. 

To justify this choice, we first establish in \textbf{Proposition~\ref{theorem:mdm_and_mdm_prime_bound}} that the optimal MDM-Prime loss is monotonically non-increasing in $\ell$:
\begin{proposition}
\label{theorem:mdm_and_mdm_prime_bound} 
Let $\ell_1,\ell_2$ be token granularities satisfying $1 < \ell_1 < \ell_2$ and $\frac{\ell_2}{\ell_1}\in\N$. The following inequalities hold:
\begin{equation}
\label{eq:non_increasing_loss}
\begin{aligned}
\inf_{p} \mathcal{L} \geq \inf_{p_{\ell_1}} \mathcal{L}^{(\ell_1)} \overset{}{\geq} \inf_{p_{\ell_2}} \mathcal{L}^{(\ell_2)}.
\end{aligned}
\end{equation}
\end{proposition}

According to \textbf{Proposition~\ref{theorem:mdm_and_mdm_prime_bound}}, increasing $\ell$ leads to lower loss value (Eq.~(\ref{eq:prime_elbo})). To validate this trend on practical applications, we conduct an experiment by varying the token granularity $\ell$ and evaluate downstream accuracy on the ARC-e (easy)~\cite{Clark2018ThinkYH} benchmark (see Fig.~\ref{fig:acc_loss}). We observe that the downstream performance of MDM-based approaches is correlated with loss: lower loss consistently leads to better downstream performance, suggesting that the loss serves as a reliable proxy for downstream task quality. In Section~\ref{sec:experiment:1B_model}, we further show that this phenomenon is applicable when the model scales up. Compared to the scaled up MDM model with $\ell=1$ (SMDM~\cite{nie2025scalingmaskeddiffusionmodels}), MDM-Prime-v2 ($\ell=\lceil\log_2 V\rceil$) exhibits improvement to the average accuracy on eight zero-shot benchmarks. Table~\ref{tab:experiment:zeroshot_qa_170} in Appendix~\ref{sec:apx:experiments:ablation} further provides a detailed ablation study on both Techniques 1 and 2 for downstream tasks. Based on the above findings, we propose the following selection principle for $\ell$:

\vspace{0.2em}
\begin{mdframed}[linewidth=0.5pt, innerleftmargin=5pt, innerrightmargin=5pt, innertopmargin=5pt, innerbottommargin=5pt]
\textbf{Technique 2.} (Binary Encoding) Select $\ell=\lceil\log_2 V\rceil$ such that $\f_\text{base-$b$}$ encodes tokens into binary sub-tokens.
\end{mdframed}

In summary, Techniques 1 and 2 suggest the following subtokenizer: $\f_\ell= \f_\text{base-$b$}\circ \f_\text{shuffle}$, where $\f_{\text{shuffle}}$ maps the original token indices into shuffled ones, while $\f_\text{base-$b$}$ performs binary encoding. The entire operation is implemented using lookup tables (Fig.~\ref{fig:subtokenizer}), requiring zero FLOPs and can be performed during data preprocessing. Further specifications are available in Appendix~\ref{sec:apx:specification:subtokenizer}.

\begin{figure*}[ht!]
    \centering
    \vspace{-0.75em}
    \includegraphics[width=\linewidth]{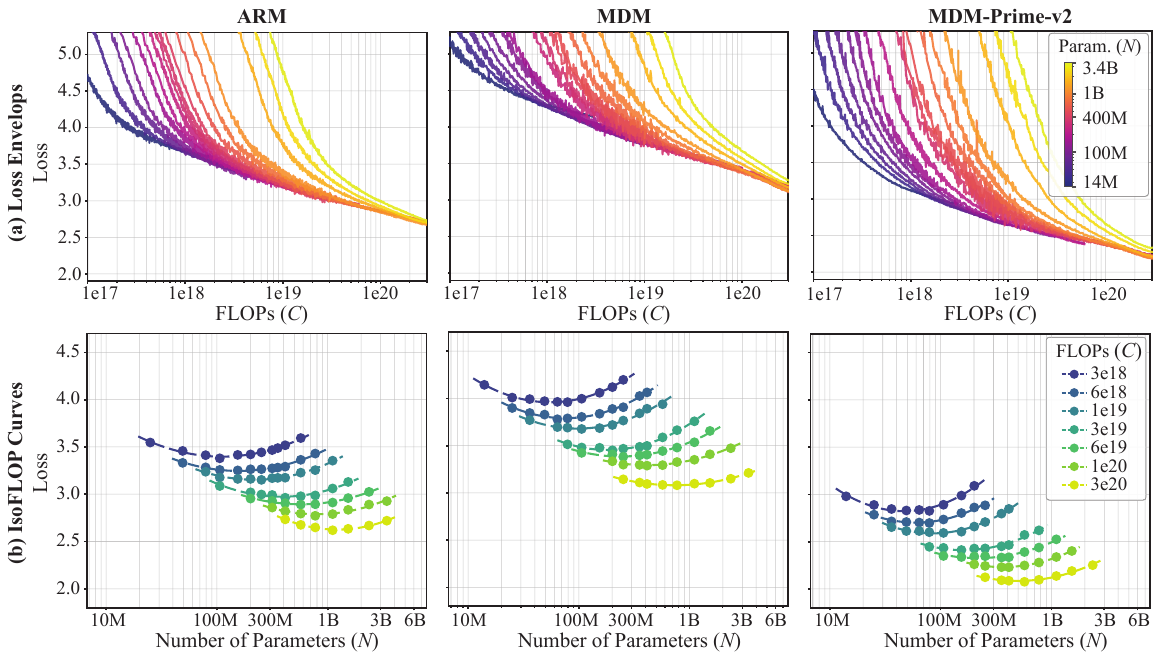}
    \vspace{-1.75em}
    \caption{The (a) loss envelopes and (b) isoFLOP curves of ARM, MDM, and MDM-Prime-v2.  In subplots (a), the number of parameters ($N$) ranges from 14M (purple) to 3.4B (yellow). In subplots (b), the compute budget ($C$) ranges from $3 \times 10^{18}$ FLOPs (blue) to $3 \times 10^{20}$ FLOPs (yellow). Please note that each model family optimizes a different training objective, the absolute loss values are not directly comparable across methods.}
    \label{fig:scaling}
    \vspace{-0.75em}
\end{figure*}

\section{Experiments}
\label{sec:experiment}
This section evaluates MDM-Prime-v2 via scaling analysis (Section~\ref{sec:experiment:scaling}), 1.1B-scale pretraining (Section~\ref{sec:experiment:1B_model}), and analysis on attention patterns and the singular value spectra of its projection weights (Section~\ref{sec:experiment:analyses}). Training details are provided in Appendix~\ref{sec:apx:configuration}. 
In Appendix~\ref{sec:apx:experiments}, we present experimental results on a 7B-parameter MDM-Prime-v2 pre-training run, validate the robustness of $\f_\text{shuffle}$ to random seed initialization, and discuss why subtokenization is ineffective for autoregressive models. A detailed ablation study on Techniques 1 and 2 for downstream tasks is presented in Table~\ref{tab:experiment:zeroshot_qa_170} in Appendix~\ref{sec:apx:experiments:ablation}. This section focuses on scaling behaviors and benchmark results.

\subsection{Loss Behavior and Scaling Properties}
\label{sec:experiment:scaling}
As established in~\cite{kaplan2020scalinglawsneurallanguage, hoffmann2022chinchilla}, the training loss of language models exhibits a strong correlation with the training FLOPs ($C$). This compute budget is primarily determined by two configuration factors: the total number of training tokens ($D$) and the number of non-embedding parameters ($N$). To understand how these factors influence the scaling behavior of MDM-Prime-v2, we analyze its training loss dynamics across various combinations of $N$ and $D$ under fixed compute budgets ranging from $3 \times 10^{18}$ to $3 \times 10^{20}$ FLOPs. For these experiments, we employ a LLaMA-styled architecture~\cite{touvron2023llama} (see Appendix~\ref{sec:apx:configuration:scaling} for specification).  All models are trained on C4~\cite{raffel2023exploringlimitstransferlearning} using the GPT-2 tokenizer with $V=50,257$.
\begin{table}[t]
    \renewcommand{\arraystretch}{1}
    \newcommand{\boldtoprule}{\toprule[1.2pt]}
    \newcommand{\boldmidrule}{\midrule[0.5pt]}
    \newcommand{\boldbottomrule}{\bottomrule[1.2pt]}
    \centering
    \large
    \resizebox{\linewidth}{!} {%
    \begin{tabular}{lcc}
        \boldtoprule
         Coefficient    & $\hat{a}$ ($N_\text{opt}\propto C^{\hat{a}}$) & $\hat{b}$ ($D_\text{opt}\propto C^{\hat{b}}$) \\
        \boldmidrule
        ARM                  & 0.52 & 0.48 \\
        MDM                  & 0.50 & 0.50 \\
        MDM-Prime-v2 ($\ell$=16) & 0.50 & 0.50 \\
        \boldbottomrule
    \end{tabular}}
    \vspace{-0.5em}
    \caption{The coefficients ($\hat{a}$ and $\hat{b}$) derived using the Chinchilla scaling law. A larger $\hat{a}$ indicates that compute resources should prioritize model parameters, while a larger $\hat{b}$ indicates that compute resources should prioritize training tokens. The other coefficients ($\alpha, \beta, A,B$, and $E$) are provided in Table~\ref{table:scaling_parameters} in Appendix.}
    \label{tab:experiment:scaling_parameters}
    \vspace{-0.75em}
\end{table}


\begin{table*}[t]
    \renewcommand{\arraystretch}{1.05}
    \newcommand{\boldtoprule}{\toprule[1.5pt]}
    \newcommand{\boldbottomrule}{\bottomrule[1.5pt]}
    \centering
    \large
    \vspace{-0.75em}
    \resizebox{\linewidth}{!}{%
    \begin{tabular}{lcccccccc|c}
        \boldtoprule
                                           & SciQ & SocialIQA & McTaco & TruthfulQA & BoolQ & ANLI & ARC-e  & OBQA & Avg. \\
        \midrule
        \multicolumn{8}{l}{\textit{Reference Models}}  \\
        \midrule
        GPT-Neo (1.3B)~\cite{gpt-neo} & 77.10 & 41.25 & 42.89  & 23.13  & 61.99 & 33.13 & 50.21 & 33.60 &  45.41\\
        OPT (1.3B)~\cite{zhang2022optopenpretrainedtransformer}     & 76.70 & 40.63 & 37.08  & 23.75  & 57.83 & 33.86 & 50.97 & 33.40 & 44.28 \\
        Pythia (1.4B)~\cite{biderman2023pythia}  & 79.20 & 40.94  & 54.25 & 22.77  & \textbf{63.15} & 33.29 & \textbf{53.91} &  33.20 & 47.59 \\
        Bloom (1.1B)~\cite{workshop2023bloom176bparameteropenaccessmultilingual} & 74.60 & 39.05 & 53.63  & 25.58   & 59.08 & 33.35 & 45.41 & 29.40 & 45.01 \\
         \midrule
        \multicolumn{8}{l}{\textit{Aligned Models}} \\
        \midrule
        SMDM (1.1B)~\cite{nie2025scalingmaskeddiffusionmodels}  & 81.20 & 41.04 & 35.07 & 24.60  & 62.17 & 32.81 & 46.13 & 33.40 & 44.55 \\
        TinyLLaMA (1.1B)~\cite{zhang2024tinyllama} & 80.90 & 39.56 & 40.88 & 20.93  & 59.20 & 33.25 & 50.13 & 33.40  & 44.78 \\
        MDM-Prime-v2 (1.1B) \textbf{(Ours)} & \textbf{83.30} & \textbf{42.02} & \textbf{66.14}  & \textbf{25.83}   & 62.05 & \textbf{34.24} & 47.81 & \textbf{34.00} & \textbf{49.42} \\
        \boldbottomrule
    \end{tabular}}
    \vspace{-0.45em}
    \caption{Zero-shot accuracy on eight commonsense reasoning benchmarks, where higher values indicate better performance. SMDM uses 1,024 Monte Carlo steps to approximate the score of each answer option, whereas MDM-Prime-v2 uses the same number of steps or fewer, as summarized in Table~\ref{tab:apx:step_setup}. Methods under `Aligned Models' share an identical training setup (dataset, compute budget, architecture, and optimization hyperparameters). Methods under `Reference Models' are similarly sized models trained under different setups and are included for reference only.}
    \vspace{-0.45em}
    \label{tab:experiment:zeroshot_qa}
\end{table*}
Figs.~\ref{fig:scaling}~(a) and (b) present the loss envelopes and isoFLOP curves for ARM, MDM, and MDM-Prime-v2. As shown in Fig.~\ref{fig:scaling}~(a), the training loss for all three methods decreases consistently as the total compute budget increases. This confirms that the loss of each method scales effectively with training FLOPs. Fig.~\ref{fig:scaling}~(b) compares the isoFLOP contours across fixed compute budgets. We note that because each model family optimizes a different training objective, the absolute loss values are not directly comparable across methods; we therefore focus on the scaling behavior within each family. We analyze the loss behavior using the Chinchilla law~\cite{hoffmann2022chinchilla}. Using our empirical observations, consisting of (loss, $N$, $D$) triplets, we fit the power-law loss estimator: $\hat{\mathcal{L}}(N,D) = E + \frac{A}{N^\alpha} + \frac{B}{D^\beta}$, where $\alpha,\beta,A,B$, and $E$ are coefficients determined via regression. Under a fixed compute budget $C \approx 6ND$, the optimal allocation of parameters ($N_{\text{opt}}$) and tokens ($D_{\text{opt}}$) is derived as follows:
\vspace{-0.35em}
\begin{equation}
N_{\text{opt}} = G \left( \frac{C}{6} \right)^{\hat{a}}, \quad D_{\text{opt}} = G^{-1} \left( \frac{C}{6} \right)^{\hat{b}},
\end{equation}
where $G = \left( \frac{\alpha A}{\beta B} \right)^{\frac{1}{\alpha+\beta}}$, $\hat{a} = \frac{\beta}{\alpha+\beta}$, and $\hat{b} = \frac{\alpha}{\alpha+\beta}$. As shown in Table~\ref{tab:experiment:scaling_parameters}, ARM exhibits the largest $\hat{a}$ and the smallest $\hat{b}$, indicating that the compute-optimal configuration of ARM prioritizes increasing model capacity ($N$) over data volume ($D$). In contrast, MDM and MDM-Prime-v2 yield the smallest $\hat{a}$ and the largest $\hat{b}$, suggesting that its compute-optimal performance is driven more by increasing training tokens than by expanding model parameters. This property is practically appealing: as current pretraining recipes (e.g., LLaMA~\cite{touvron2023llama}) increasingly focus on scaling smaller models through longer training for efficient inference, a framework that achieves small $\hat{a}$ and large $\hat{b}$ offers a path to improved capabilities. These coefficients determine the compute-optimal frontier for each method and serve as a practical guide for configuring MDM-Prime-v2 under a compute budget. We offer further scaling guidelines in Appendix~\ref{sec:apx:experiments:scaling}.

\begin{table}[t]
    \renewcommand{\arraystretch}{1}
    \newcommand{\boldtoprule}{\toprule[2.3pt]}
    \newcommand{\boldmidrule}{\midrule[0.8pt]}
    \newcommand{\boldbottomrule}{\bottomrule[2.3pt]}
    \centering
    \huge
    \resizebox{\linewidth}{!} {%
    \begin{tabular}{lcccc}
        \boldtoprule
         & \multicolumn{2}{c}{SMDM (1.1B)} & \multicolumn{2}{c}{MDM-Prime-v2 (1.1B)} \\
        Steps & Gen PPL ($\downarrow$) & Entropy ($\uparrow$) & Gen PPL ($\downarrow$) & Entropy ($\uparrow$) \\
        \boldmidrule 
        128  & 129.09 &  4.55  & 122.80 &  4.60 \\
        256  & 118.00 &  4.36  & 109.08 &  4.36  \\
        512  & 86.08  &  4.23  & 79.82  & 4.24  \\
        768  & 81.01  &  3.92  & 66.33 & 3.99  \\
        1024 & 77.24  &  3.77  & 53.06 & 3.79  \\
        \boldbottomrule
    \end{tabular}}
    \vspace{-0.2em}
    \caption{Sample quality comparison between SMDM (1.1B) and MDM-Prime-v2 (1.1B) across varying sampling steps, measured by Gen PPL and entropy. MDM-Prime-v2 consistently achieves lower Gen PPL with comparable or higher entropy, with the quality gap widening at larger step counts.}
    \label{tab:experiment:sample_quality}
\end{table}

\subsection{Improvement at Larger-Scale Pretraining}
\label{sec:experiment:1B_model}
\textbf{Commonsense Reasoning.}~In this experiment, we adopt the training configuration of TinyLLaMA (ARM)~\cite{zhang2024tinyllama} and SMDM (MDM)~\cite{nie2025scalingmaskeddiffusionmodels} to train a 1.1B parameter model on 540B tokens from the Slimpajama dataset~\cite{cerebras2023slimpajama} (totaling $3.3\times 10^{21}$ FLOPs). We compare the models on a wide range of commonsense reasoning tasks, with the descriptions of each task shown in Table~\ref{tab:apx:commonsense} in Appendix. The model architecture and tokenizer are based on LLaMA.

Table~\ref{tab:experiment:zeroshot_qa} compares MDM-Prime-v2 against several pretrained ARM and MDM baselines of similar size: GPT-Neo (1.3B), OPT (1.3B), Pythia (1.4B), Bloom (1.1B), SMDM (1.1B), and TinyLLaMA (1.1B). MDM-Prime-v2 achieves the highest average accuracy across the tasks, outperforming these baselines on six of the eight tasks. In particular, our model demonstrates advantages in temporal reasoning and scientific question answering, delivering noticeable gains on SciQ and McTaco.

\begin{figure*}[t!]
    \centering
    \includegraphics[width=\linewidth]{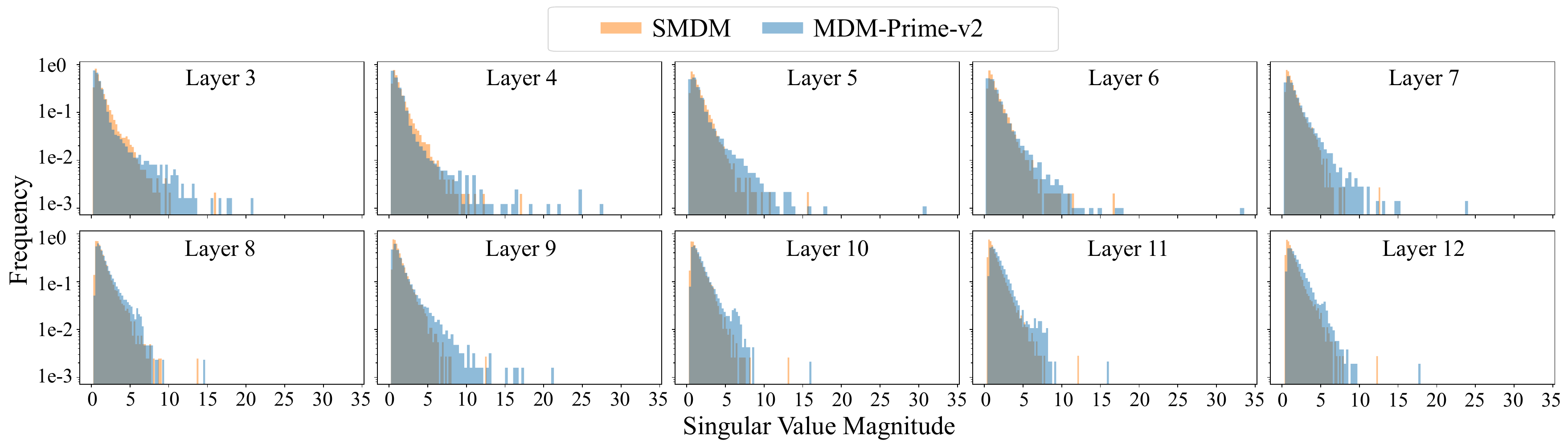}
    \vspace{-1.895em}
    \caption{Singular value distributions of the projection matrices of SMDM and MDM-Prime-v2. Both models adopt the same architecture containing 20 layers. A complete plot of all layers is shown in Fig.~\ref{fig:apx:spectra}.}
    \vspace{-0.5em}
    \label{fig:apx:spectra_part}
\end{figure*}

\begin{figure*}[h!]
    \centering
    \includegraphics[width=\linewidth]{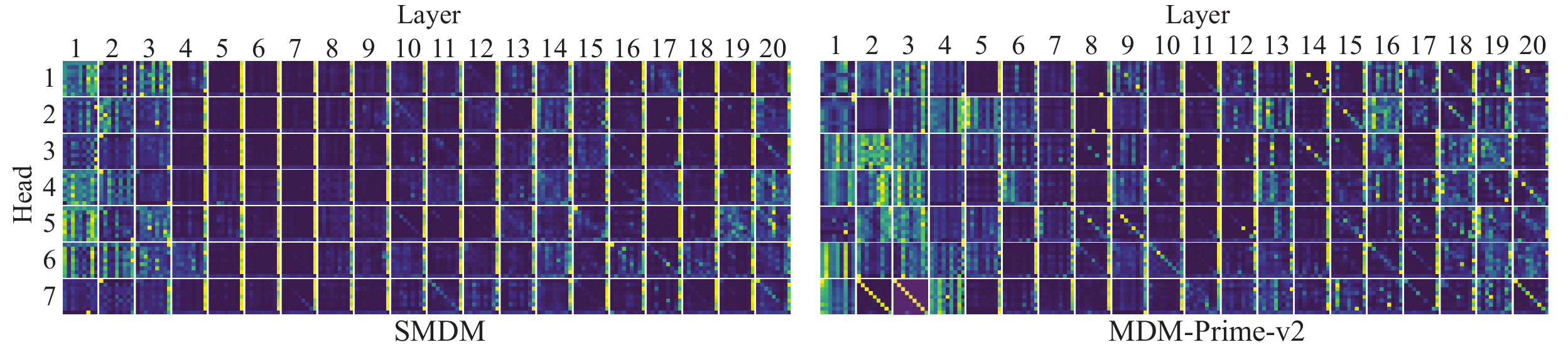}
    \vspace{-1.8em}
    \caption{Attention score patterns of SMDM and MDM-Prime-v2. Both models contain 20 layers and 14 attention heads. A complete plot of all heads is shown in Fig.~\ref{fig:apx:attention}. MDM exhibits vertical stripes across Layers 4–20. In contrast, MDM-Prime-v2 displays diagonal structures and a richer diversity of patterns. }
    \vspace{-1em}
    \label{fig:apx:attention_part}
\end{figure*}

\textbf{Sample Quality.}~We compare the sample quality of SMDM and MDM-Prime-v2 using generative perplexity (Gen PPL) and entropy as metrics to assess the fidelity-diversity trade-off. Gen PPL is evaluated using TinyLLaMA trained on 3T tokens. As shown in Table~\ref{tab:experiment:sample_quality}, MDM-Prime-v2 consistently achieves lower Gen PPL with equal or higher entropy, demonstrating improved sample quality over SMDM. Furthermore, we observe that this improvement becomes more pronounced as the number of sampling steps increases, evidenced by the larger Gen PPL gap at 768 and 1024 steps. 


\subsection{Analyses on Attention Sink Pattern and Singular Value Spectra}
\label{sec:experiment:analyses}


Analyzing attention score patterns~\cite{clark2019doesbertlookat, xiao2024efficientstreaminglanguagemodels} and the singular value spectra of self-attention weights~\cite{martin2019traditionalheavytailedselfregularization, dong2021attentionneedpureattention} serves as a diagnostic tool for understanding model capacity utilization and trainability. In this section, we examine the attention mechanisms and spectral properties of SMDM and MDM-Prime-v2 to elucidate the factors driving the latter's superior performance.

We observe a divergence in the attention score patterns between SMDM and MDM-Prime-v2. Fig.~\ref{fig:apx:attention} provides an illustration of the difference. SMDM exhibits signs of \textit{Attention Sink}~\cite{xiao2024efficientstreaminglanguagemodels}, which is characterized by a single vertical stripe in the attention score matrices. This indicates that a significant proportion of heads attend to specific hidden representations. These representations function as learned `no-op' registers, absorbing excess attention probability when no relevant context is present, which ultimately results in a lack of specialized routing for the attention operation. In contrast, MDM-Prime-v2 displays a richer diversity of attention patterns with more attention score matrices exhibiting sharp diagonal lines. These patterns suggest that MDM-Prime-v2 learns more specialized routing mechanisms, which potentially leads to improved representational capacity.

Spectral analysis of query-key-value projection matrices ($\mathbf{W}_\text{qkv}$) serves as an indicator of the utilization of model capacity~\cite{martin2019traditionalheavytailedselfregularization, dong2021attentionneedpureattention}. The singular value spectra for SMDM and MDM-Prime-v2 are presented in Fig.~\ref{fig:apx:spectra_part}. MDM-Prime-v2 exhibits a heavier-tailed distribution compared to SMDM. This slower spectral decay indicates a reduction in rank collapse. Quantitatively, this is measured by the stable rank ($\|\mathbf{W}_\text{qkv}\|_F^2/\|\mathbf{W}_\text{qkv}\|_2^2$): MDM-Prime-v2 achieves a stable rank of 10.0, surpassing SMDM's 8.3. The result suggests that MDM-Prime-v2 has improved ability to capture diverse features rather than over-focusing on a single dominant feature.








\section{Conclusion}
We present MDM-Prime-v2, an enhanced and scaled-up version of MDM-Prime. Our analysis reveals the sub-optimality of base-$b$ encoding for $\f_\ell$ with BPE tokenizers, and identifies a critical connection between the token granularity $\ell$ and the objective. We resolve these issues by employing an index shuffling operation in the subtokenizer design and selecting $\ell=\lceil\log_2 V\rceil$. We characterize the scaling behavior of MDM-Prime-v2 across varying model and data scales, and demonstrate its practical effectiveness at the 1.1B-parameter scale, where it achieves improvements across eight zero-shot reasoning benchmarks. Our analysis of attention score patterns and singular value distributions further reveals that MDM-Prime-v2 learns more specialized routing mechanisms and exhibits larger stable ranks in its linear weights, indicating improved representational capacity.



\section*{Limitations}
MDM-Prime-v2 adopts a random token index shuffling operation as a simple implementation to address distributional issues caused by the BPE tokenizer. However, \textbf{Proposition~\ref{proposition:entropy_max}} suggests that the optimal design is a function that maximizes the entropy of $\vy_t$. Although our experimental results demonstrate that index shuffling effectively increases entropy, yielding a near-optimal design, we anticipate that a rigorously derived algorithmic approach could further enhance performance. As an ablation, we compare index shuffling against a greedy assignment algorithm that maximizes entropy given corpus token-frequency statistics; results are reported in Appendix~\ref{sec:apx:discussion}. Greedy assignment achieves marginally higher entropy but brings negligible performance improvement. Given that it requires additional corpus statistics whereas index shuffling is fully off-the-shelf, we recommend index shuffling for practical use. Nevertheless, designing entropy-maximizing assignment strategies for subtokenizer represents a promising direction for future work to enhance MDM-Prime-v2, and Appendix~\ref{sec:apx:discussion} discusses two further complementary avenues for extending this work.

\section*{Acknowledgements}
RGK is supported by a Canada CIFAR AI Chair and a Canada Research Chair Tier II in Computational Medicine. This research was supported by an NFRF Special Call NFRFR2022-00526. Resources used in preparing this research were provided, in part, by the Province of Ontario, the Government of Canada through CIFAR, and companies sponsoring the Vector Institute. The authors gratefully acknowledge the support from the National Science and Technology Council (NSTC) in Taiwan under grant numbers NSTC 114-2221-E-002-069-MY3, NSTC 113-2221-E-002-212-MY3, NSTC 115-2634-F-002-012, and NSTC 115-2218-E-A49-010, as well as the support from the Academia Sinica Scholar Award (ASSA) under grant number AS-ASSA-115-02, NTU Artificial Intelligence Center of Research Excellence, and Taiwan Centers of Excellence in Artificial Intelligence. This research was also supported by the NVIDIA Academic Grant Program. The authors would also like to express their appreciation for the hardware grant donation of the GPUs from NVIDIA Corporation and NVIDIA AI Technology Center (NVAITC) used in this work. Furthermore, the authors extend their gratitude to the National Center for High-Performance Computing (NCHC) for providing computational and storage resources. The authors also thank the NVIDIA Taipei-1 supercomputer for providing essential computing resources. Finally, we are grateful to Viet Nguyen, Xiaowen Zhang, and Vahid Belazadeh Meresht for their thoughtful review and valuable feedback on this work.



\bibliography{citation}

\appendix



\setcounter{section}{0}
\setcounter{equation}{0}
\setcounter{figure}{0}
\setcounter{table}{0}

\renewcommand{\thefigure}{A\arabic{figure}}
\renewcommand{\thetable}{A\arabic{table}}
\renewcommand{\theequation}{A\arabic{equation}}

\newpage
\appendix
\onecolumn
\section{Appendix}

This appendix provides additional discussions and experiments. Section~\ref{sec:apx:proof} presents theoretical analyses. Section~\ref{sec:apx:specification} offers technical implementation details of the model architecture and subtokenizer of MDM-Prime-v2. Section~\ref{sec:apx:configuration} elaborates on the experimental configurations. Section~\ref{sec:apx:experiments} reports additional experimental results. Section~\ref{sec:apx:related_works} walks through some related studies. Finally, Section~\ref{sec:apx:discussion} outlines limitations and directions for future work.


\subsection{Theoretical Analyses}
\label{sec:apx:proof}
In this section, we present theoretical derivations. In Section~\ref{sec:apx:proof:mdm_and_mdm_prime_bound}, we present the proofs of \textbf{Proposition~\ref{theorem:mdm_and_mdm_prime_bound}}, while Section~\ref{sec:apx:maximum_ent_subtokenizer} establishes the justifications for \textbf{Propositions~\ref{proposition:decomposition}} and \textbf{\ref{proposition:entropy_max}}. A table of symbols used in this paper is presented in Table~\ref{tab:symbols}.


\begin{table}[h]
    \renewcommand{\arraystretch}{1.075}
    \newcommand{\boldtoprule}{\toprule[1.2pt]}
    \newcommand{\boldbottomrule}{\bottomrule[1.2pt]}
    \centering
    \resizebox{\textwidth}{!}{%
    \begin{tabular}{cl|cl}
        \boldtoprule
         Symbol & Definition & Symbol & Definition \\
        \hline
         $L$ & token sequence length	& $V$ & number of classes a token can take \\
         $\ell$ & sub-token sequence length	& $b$ & number of classes a sub-token can take \\
         $\X$ & a set of tokens: $\{0,\cdots,V-1\}$	& $\tilde{\X}$ & an augmented set of tokens: $\X\cup \{\m\}$ \\
         $\Y$ & a set of sub-tokens: $\{0,\cdots,b-1\}$	& $\tilde{\Y}$ & an augmented set of sub-tokens: $\Y\cup \{\m\}$ \\
         $\vx_0$ & 	a sequence of clean tokens ($\vx_0\in\X^L$)	& $\vx_t$ & a sequence of noised tokens ($\vx_t\in\tilde{\X}^L$) \\
         $\vy_0$ & 	a sequence of clean sub-tokens ($\vy_0\in\Y^{L\times \ell}$)	& $\vy_t$ & a sequence of noised sub-tokens ($\vy_t\in\tilde{\Y}^{L\times \ell}$) \\
         $\delta_{x}(x')$ &  Kronecker delta (equals 1 only if $x=x'$) & $\alpha_t$ & scheduling function ($\alpha_t:[0,1]\to[0,1]$) \\
         $\f_{\ell}(\vx_0)$ &  invertible base-$b$ encoding & $\g_\ell(\vy_t)$ & a function that maps between latent variables \\
         $p(\vx_0|\vx_t)$ &  MDM model (\textit{carry-over} distribution of $\vx_0$) & $p_\ell(\vy_0|\vy_t)$ & MDM-Prime model (\textit{carry-over} distribution of $\vy_0$) \\
         $q_\text{data}(\vx_0)$ & data distribution of $\vx_0$ & $q_\text{data\_y}(\vy_0)$ & data distribution of $\vy_0$ ($q_\text{data\_y}=q_\text{data}\circ\f^{-1}_\ell$) \\
        \boldbottomrule
    \end{tabular}
    }\vspace{-0.5em}
    \caption{A table of symbols.}
    \vspace{-0.5em}
    \label{tab:symbols}
\end{table}

\subsubsection{Proof of Propositions 3.3}
\label{sec:apx:proof:mdm_and_mdm_prime_bound}

We present our main findings (\textbf{Proposition~\ref{theorem:mdm_and_mdm_prime_bound}} along with its equality condition) as a merged \textbf{Theorem~\ref{theorem:apx:overall_bound}}. The proof relies on a series of auxiliary results provided in \textbf{Lemmas \ref{lemma:expectation_equality}-\ref{lemma:invariant_schedule}} and \textbf{Propositions~\ref{proposition:x_y_bound}-\ref{proposition:y_y_bound}}.

\begin{theorem}
\label{theorem:apx:overall_bound}
Let $p$ and $p_\ell$ denote the MDM and MDM-Prime models. Let $\ell_1,\ell_2$ be token granularities satisfying $1 < \ell_1 < \ell_2$ and $\frac{\ell_2}{\ell_1}\in\N$. The following inequalities hold:
\begin{equation}
\label{eq:apx:bound}
\begin{aligned}
\inf_{p} \mathcal{L}  \geq \inf_{p_{\ell_1}} \mathcal{L}^{(\ell_1)}  \geq \inf_{p_{\ell_2}} \mathcal{L}^{(\ell_2)} .
\end{aligned}
\end{equation}
Let $\tilde{\ell}\in\N$ and $\g_{\tilde{\ell}}$ be a vectorized mapping with element-wise operation defined as:
\begin{equation}
\label{eq:apx_coarsening}
    g_{\tilde{\ell}}(\vy^i_t)=
    \begin{cases}
      f^{-1}_{\tilde{\ell}}(\vy^i_t), & \text{if }  y^{i,j}_t\neq \texttt{m}, \forall j\in\{1,...,\tilde{\ell}\},\\
      \texttt{m}, & \text{otherwise},
    \end{cases}
\end{equation}
Given a scheduling function $\alpha_t$ and defining $\tilde{\alpha}_t=\alpha_t^{1/\tilde{\ell}}$, the first inequality in Eq.~(\ref{eq:apx:bound}) becomes an equality if and only if:
\begin{equation}
\begin{aligned}
\label{eq:apx:kl_condition}
\E_{q_{\tilde{\alpha}}(\vy_t)}\!\Big[\D_\mathrm{KL}\!\big(q_{\rm{data\_y}}(\vy_0 \mid \vy_t)\, \| \, q_{\rm{data\_y}}(\vy_0\mid \g_{\tilde{\ell}}(\vy_t))\big)\Big] =0,
\end{aligned}
\end{equation}
where $\vy_0\in \Y^{L\times\ell_1}$, $\vy_t\in \tilde{\Y}^{L\times\ell_1}$, $\tilde{\ell}=\ell_1$, and the second inequality becomes an equality if and only if Eq.~(\ref{eq:apx:kl_condition}) holds with $\tilde{\ell}=\frac{\ell_2}{\ell_1}$, $\vy_0\in \Y^{L\times\ell_2}$, and $\vy_t\in \tilde{\Y}^{L\times\ell_2}$. 
\end{theorem}

\begin{proof}
The first inequality $\inf_{p} \mathcal{L}  \geq \inf_{p_{\ell_1}} \mathcal{L}^{(\ell_1)} $ and the corresponding equality condition (Eq.~(\ref{eq:apx:kl_condition})) are established in \textbf{Proposition~\ref{proposition:x_y_bound}} while the second inequality $\inf_{p_{\ell_1}}\mathcal{L}^{(\ell_1)}  \geq \inf_{p_{\ell_2}} \mathcal{L}^{(\ell_2)} $ and the corresponding equality condition are established in \textbf{Proposition~\ref{proposition:y_y_bound}}. Combining these two results yields the chain of inequalities presented in Eq.~(\ref{eq:apx:bound}).
\end{proof}

The proof of \textbf{Proposition~\ref{proposition:x_y_bound}} is organized into three primary steps, as illustrated in Fig.~\ref{fig:apx:derivation_steps}:
\begin{itemize}
\vspace{-0.5em}
    \item \textbf{Step 1}: We identify a scheduling function $\tilde{\alpha}_t$ and a mapping $\g_\ell$ such that the expectations satisfy $\E_{q_\alpha(\vx_0, \vx_t)}[F(\vx_0, \vx_t)] = \E_{q_{\tilde{\alpha}}(\vy_0, \vy_t)}[F(\f^{-1}_\ell(\vy_0), \g_\ell(\vy_t))]$ for any arbitrary function $F: \mathcal{X}^{L} \times \tilde{\mathcal{X}}^{L} \to \mathbb{R}$. This equivalence is formally established in \textbf{Lemma~\ref{lemma:expectation_equality}}. \vspace{-0.25em}
    \item \textbf{Step 2}: We demonstrate that the logarithmic probability of MDM-Prime is point-wisely greater than or equal to that of MDM for every timestep $t \in [0,1]$. This comparison is detailed in \textbf{Lemma~\ref{lemma:log_likelihood_greater}}.  \vspace{-0.25em}
    \item \textbf{Step 3}: Finally, we prove that the training objective remain invariant under different choices of the scheduling function, as presented in \textbf{Lemma~\ref{lemma:invariant_schedule}}.
    \vspace{-0.5em}
\end{itemize}
By synthesizing these three results, we arrive at the inequality established in the proposition.

\begin{figure}[t]
    \centering
    \includegraphics[width=0.925\linewidth]{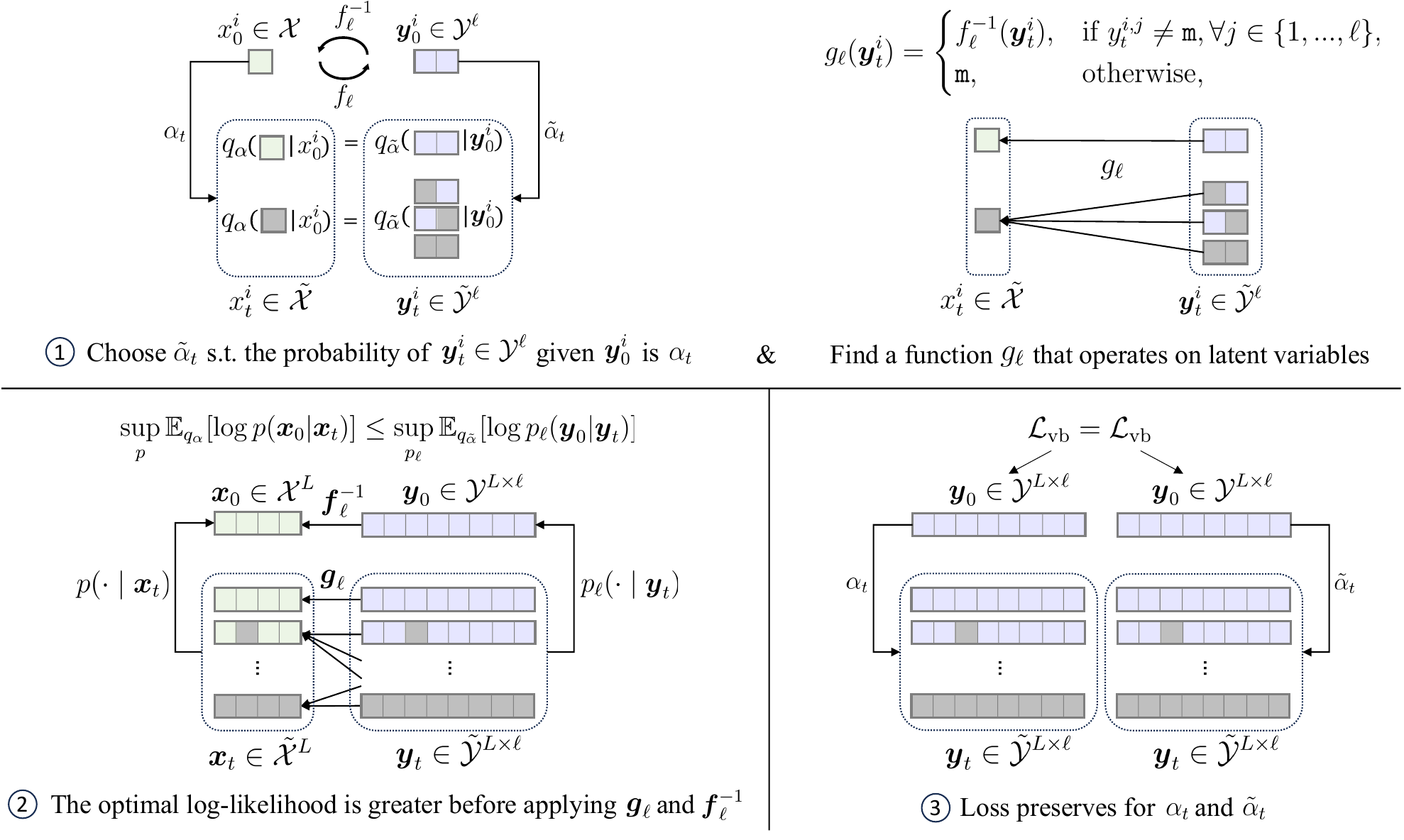}
    \vspace{-0.25em}
    \caption{An illustration of the three key steps for proving \textbf{Proposition}~\ref{proposition:x_y_bound}. Gray tiles represent masked tokens or sub-tokens. Green and purple tiles represent unmasked tokens and sub-tokens, respectively.}
    \label{fig:apx:derivation_steps}
    \vspace{-0.5em}
\end{figure}




\begin{lemma}
\label{lemma:expectation_equality}
Let $\vx_t\sim q_\alpha(\cdot \mid \vx_0)$ and $\vy_t\sim q_{\tilde{\alpha}}(\cdot \mid \vy_0)$, where $\tilde{\alpha}_t\triangleq\alpha_t^{1/\ell}$ and $\vy_0=\f_\ell(\vx_0)$. In addition, let $\g_\ell:\tilde{\Y}^{L\times \ell}\to \tilde{\X}^{L}$ be a function defined in Eq.~(\ref{eq:apx_coarsening}). Given any mapping function $F:\X^{L}\times \tilde{\X}^{L}\to \R$, the following equality holds:
\begin{equation}
\label{eq:apx:exp_equal}
    \E_{q_\alpha(\vx_0, \vx_t)}[F(\vx_0, \vx_t)] = \E_{q_{\tilde{\alpha}}(\vy_0, \vy_t)}[F(\f^{-1}_\ell(\vy_0), \g_\ell(\vy_t))].
\end{equation}
\end{lemma}

\begin{proof}
\begin{equation*}
\begin{aligned}
&\E_{q_{\tilde{\alpha}}(\vy_0, \vy_t)}[F(\f^{-1}_\ell(\vy_0), \g_\ell(\vy_t))]\\
&\stackrel{(i)}{=}\E_{q_\text{data}\circ \f^{-1}_\ell(\vy_0)} \left[ \sum_{\vy_t}  F(\f^{-1}_\ell(\vy_0), \g_\ell(\vy_t)) q_{\tilde{\alpha}}(\vy_t \mid\vy_0) \right ]\\
&=\E_{q_\text{data}\circ \f^{-1}_\ell(\vy_0)} \left[ \sum_{\vy_t}  F(\f^{-1}_\ell(\vy_0), \g_\ell(\vy_t))\prod_{i=1}^{L}\prod_{j=1}^{\ell} q_{\tilde{\alpha}}(y^{i,j}_t\mid y_0^{i,j}) \right ]\\
&\stackrel{(ii)}{=}  \E_{q_\text{data}\circ \f^{-1}_\ell(\vy_0)} \left[ \sum_{\vy_t} F(\f^{-1}_\ell(\vy_0), \g_\ell(\vy_t)) \prod_{i=1}^{L}\prod_{j=1}^{\ell} (1-\tilde{\alpha}_t)\delta_{\texttt{m}}(y^{i,j}_t)+\tilde{\alpha}_t \delta_{y_0^i}(y^{i,j}_t)\right ] \quad\quad\quad\quad\quad\quad\quad\quad\quad\quad\quad\quad\quad\quad\quad
\end{aligned}
\end{equation*}
\begin{equation*}
\begin{aligned}
&= \E_{q_\text{data}\circ \f^{-1}_\ell(\vy_0)} \left[ \sum_{\vy_t} F(\f^{-1}_\ell(\vy_0), \g_\ell(\vy_t))  \prod_{i=1}^{L}\prod_{j=1}^{\ell} (1-\alpha^{1/\ell}_t)\delta_{\texttt{m}}(y^{i,j}_t)+\alpha^{1/\ell}_t \delta_{y_0^{i,j}}(y^{i,j}_t)\right ] \\
&=\E_{q_\text{data}\circ \f^{-1}_\ell(\vy_0)} \Bigg[ \sum_{\vy_t} F(\f^{-1}_\ell(\vy_0), \g_\ell(\vy_t)) \prod_{i=1}^{L} (1-(\alpha^{1/\ell}_t))^{\ell}\delta_{\texttt{m}}(y^{i,1}_t)\cdots\delta_{\texttt{m}}(y^{i,\ell}_t) +
\cdots \\
&\quad\quad\quad\quad\quad\quad\quad\quad\quad\quad\quad\quad\quad\quad\quad\quad\quad\quad\quad\quad\quad\quad\quad\quad\quad\quad\quad + (\alpha^{1/\ell}_t)^{\ell} \delta_{y_0^{i,1}}(y^{i,1}_t)\cdots\delta_{y_0^{i,\ell}}(y^{i,\ell}_t) \Bigg ] \\
&=\E_{q_\text{data}\circ \f^{-1}_\ell(\vy_0)} \Bigg[ \sum_{\vy_t} F(\f^{-1}_\ell(\vy_0), \g_\ell(\vy_t)) \prod_{i=1}^{L} (1-(\alpha^{1/\ell}_t))^{\ell}\delta_{\texttt{m},\cdots,\texttt{m}}(y^{i,1}_t,\cdots,y^{i,\ell}_t) +\cdots \\
&\quad\quad\quad\quad\quad\quad\quad\quad\quad\quad\quad\quad\quad\quad\quad\quad\quad\quad\quad\quad\quad\quad\quad\quad\quad\quad\quad + (\alpha^{1/\ell}_t)^{\ell} \delta_{y^{i,1}_t,\cdots,y_0^{i,\ell}}(y^{i,1}_0,\cdots,y^{i,\ell}_t) \Bigg ] \\
&=\E_{q_\text{data}\circ \f^{-1}_\ell(\vy_0)} \left[ \sum_{\vy_t} F(\f^{-1}_\ell(\vy_0), \g_\ell(\vy_t)) \prod_{i=1}^{L} (1-(\alpha^{1/\ell}_t))^{\ell}\delta_{\texttt{m},\cdots,\texttt{m}}(\vy^{i}_t) +\cdots + (\alpha^{1/\ell}_t)^{\ell} \delta_{\vy^{i}_0}(\vy^{i}_t) \right ] \\
&=\E_{q_\text{data}(\vx_0)} \left[ \sum_{\vy_t} F(\vx_0, \g_\ell(\vy_t)) \prod_{i=1}^{L} \underbrace{(1-(\alpha^{1/\ell}_t))^{\ell}\delta_{\texttt{m},\cdots,\texttt{m}}(\vy^{i}_t) +\cdots}_{\text{Contains at least one \texttt{m}}} + \underbrace{(\alpha^{1/\ell}_t)^{\ell} \delta_{f_\ell(x^{i}_0)}(\vy^{i}_t) }_{\text{Contains no \texttt{m}}}\right ] \\
&\stackrel{(iii)}{=} \E_{q_\text{data}(\vx_0)} \left[ \sum_{\vy_t} F(\vx_0, \g_\ell(\vy_t)) \prod_{i=1}^{L} (1-(\alpha^{1/\ell}_t)^{\ell})\delta_{\texttt{m}}(g_\ell(\vy_t^i)) + (\alpha^{1/\ell}_t)^{\ell} \delta_{x_0^{i}}(g_\ell(\vy^{i}_t))\right ] \\
&= \E_{q_\text{data}(\vx_0)} \left[ \sum_{\vy_t} F(\vx_0, \g_\ell(\vy_t)) \prod_{i=1}^{L} (1-\alpha_t)\delta_{\texttt{m}}(g_\ell(\vy_t^i)) + \alpha_t \delta_{x_0^{i}}(g_\ell(\vy^{i}_t))\right ] \\
&= \E_{q_\text{data}(\vx_0)} \left[ \sum_{\vy_t} F(\vx_0, \g_\ell(\vy_t)) q_\alpha(\g_\ell(\vy_t)\mid \vx_0) \right ] \\
&\stackrel{(iv)}{=} \E_{q_\text{data}(\vx_0)} \left[ \sum_{\vx_t} F(\vx_0, \vx_t) q_\alpha(\vx_t\mid \vx_0) \right ] \\
&= \E_{q_\alpha(\vx_0, \vx_t)}[F(\vx_0, \vx_t)]. \\
\end{aligned}
\end{equation*}
In  step $(i)$, $q_{\tilde{\alpha}}(\vy_0, \vy_t)=q_{\tilde{\alpha}}(\vy_t|\vy_0)q_\text{data\_y}(\vy_0)=q_{\tilde{\alpha}}(\vy_t|\vy_0)q_\text{data}\circ\f^{-1}_\ell(\vy_0)$ by definition. In step $(ii)$, $q_{\tilde{\alpha}}(\vy_t|\vy_0)$ is expanded according to Eq.~(\ref{eq:prime_kernel}). 
In step $(iii)$, since $F$ depends on $\vy_t$ only through $g_\ell(\vy_t)$, we group terms by their image under $g_\ell$. The domain of $\vy^i_t$ partitions into vectors containing at least one $\texttt{m}$ and the fully-unmasked value $f_\ell(x_0^i)$. Their respective probability masses are $1 - (\tilde{\alpha}_t)^\ell = 1 - \alpha_t$ and $(\tilde{\alpha}_t)^\ell = \alpha_t$, yielding the two-term expression in terms of $g_\ell(\vy^i_t)$.
In step $(iv)$, the summation is performed on $\vx_t\in \tilde{\X}^L$ (i.e., the codomain of $\g_\ell$).
\end{proof}

\begin{lemma}
\label{lemma:log_likelihood_greater}
Let $\tilde{\alpha}_t\triangleq \alpha_t^{1/\ell}$. The following inequality holds:
\begin{equation}
\begin{aligned}
\sup_{p}\E_{q_{\alpha}(\vx_0,\vx_t)}\!\big[\log  p(\vx_0\mid \vx_t)\big]
& \le \sup_{p_{\ell}} \E_{q_{\tilde{\alpha}}(\vy_0,\vy_t)}\!\big[\log p_\ell(\vy_0\mid \vy_t)\big],\quad \forall t\in[0,1].
\end{aligned}
\end{equation}
The equality holds if and only if:
\begin{equation}
\begin{aligned}
\E_{q_{\tilde{\alpha}}(\vy_t)}\left[\D_\text{KL}(q_{\rm{data\_y}}(\vy_0\mid \vy_t)\,\| \,q_{\rm{data\_y}}(\vy_0 \mid \g_\ell(\vy_t))\right] = 0.
\end{aligned}
\end{equation}
\end{lemma}

\begin{proof}
The proof demonstrates that $\sup_{p_{\ell}} \E_{q_{\tilde{\alpha}}(\vy_0,\vy_t)}\!\big[\log p_\ell(\vy_0| \vy_t)\big]-\sup_{p}\E_{q_{\alpha}(\vx_0,\vx_t)}\!\big[\log  p(\vx_0| \vx_t)\big]\geq 0$ by first expanding both terms. $\sup_{p_\ell}\E_{q_{\tilde{\alpha}}(\vy_0,\vy_t)}[\log p_\ell(\vy_0| \vy_t)]$ is expanded as follows:
\begin{align*}
\sup_{p_\ell}&\E_{q_{\tilde{\alpha}}(\vy_0,\vy_t)}[\log p_\ell(\vy_0\mid \vy_t)] \\
&\stackrel{(i)}{=}\sup_{p_\ell}\left(\E_{q_{\tilde{\alpha}}(\vy_0,\vy_t)}[\log p_\ell(\vy_0\mid \vy_t)]+ \E_{q_{\tilde{\alpha}}(\vy_0,\vy_t)}[\log q_\text{data\_y}(\vy_0\mid \vy_t)] - \E_{q_{\tilde{\alpha}}(\vy_0,\vy_t)}[\log q_\text{data\_y}(\vy_0\mid \vy_t)] \right)\\
&=\sup_{p_\ell}\left(\E_{q_{\tilde{\alpha}}(\vy_0,\vy_t)}[\log q_\text{data\_y}(\vy_0\mid \vy_t) - \E_{q_{\tilde{\alpha}}(\vy_0,\vy_t)}\left[\log \frac{q_\text{data\_y}(\vy_0\mid \vy_t)}{ p_\ell(\vy_0 \mid \vy_t)}\right]\right) \\
\end{align*}
\begin{align*}
&=\sup_{p_\ell}\left(\E_{q_{\tilde{\alpha}}(\vy_0,\vy_t)}[\log q_\text{data\_y}(\vy_0\mid \vy_t) - \E_{q_{\tilde{\alpha}}(\vy_t)}\left[\E_{q_\text{data\_y}(\vy_0|\vy_t)}\left[\log \frac{q_\text{data\_y}(\vy_0\mid \vy_t)}{ p_\ell(\vy_0 \mid \vy_t)}\right]\right]\right) \\
&=\E_{q_{\tilde{\alpha}}(\vy_0,\vy_t)}[\log q_\text{data\_y}(\vy_0\mid \vy_t)] -\inf_{p_\ell} \E_{q_{\tilde{\alpha}}(\vy_t)}\!\Big[\D_\mathrm{KL}\!\big(q_\text{data\_y}(\vy_0\mid \vy_t)\,\|\, p_\ell(\vy_0\mid\vy_t)\big)\Big] \\
&\stackrel{(ii)}{=}\E_{q_{\tilde{\alpha}}(\vy_0,\vy_t)}[\log q_\text{data\_y}(\vy_0\mid \vy_t)],
\end{align*}
where $(i)$ adds and subtracts $\E_{q_{\tilde{\alpha}}(\vy_0,\vy_t)}[\log q_\text{data\_y}(\vy_0| \vy_t)]$, and $(ii)$ shows the optimum that the KL divergence becomes zero (i.e., $p_\ell(\vy_0|\vy_t)=q_\text{data\_y}(\vy_0| \vy_t)$). On the other hand, $\sup_{p}\E_{q_\alpha(\vx_0,\vx_t)}[\log p(\vx_0| \vx_t)]$ can be simplified as an expected entropy term by following a similar derivation:
\begin{align*}
\sup_{p}&\E_{q_\alpha(\vx_0,\vx_t)}[\log p(\vx_0\mid \vx_t)] \\
&=\sup_{p}\left(\E_{q_\alpha(\vx_0,\vx_t)}[\log p(\vx_0\mid \vx_t)]+ \E_{q_\alpha(\vx_0,\vx_t)}[\log q_\text{data}(\vx_0\mid \vx_t)] - \E_{q_\alpha(\vx_0,\vx_t)}[\log q_\text{data}(\vx_0\mid \vx_t)]\right) \\
&=\sup_{p}\left(\E_{q_\alpha(\vx_0,\vx_t)}[\log q_\text{data}(\vx_0\mid \vx_t)] - \E_{q_\alpha(\vx_0,\vx_t)}\left[\log\frac{ q_\text{data}(\vx_0\mid \vx_t)}{p(\vx_0\mid \vx_t)}\right]\right) \\
&=\E_{q_\alpha(\vx_0,\vx_t)}[\log q_\text{data}(\vx_0\mid \vx_t)] - \inf_{p}\E_{q_{\alpha}(\vx_t)}\!\Big[\D_\mathrm{KL}\!\big(q_\text{data}(\vx_0\mid \vx_t)\,\| \,p(\vx_0 \mid \vx_t)\big)\Big] \\
&=\E_{q_\alpha(\vx_0,\vx_t)}[\log q_\text{data}(\vx_0\mid \vx_t)]\\
&\stackrel{(i)}{=}\E_{q_{\tilde{\alpha}}(\vy_0,\vy_t)}[\log q_\text{data}(\f_\ell^{-1}(\vy_0) \mid \g_\ell(\vy_t))],
\end{align*}
where $(i)$ rewrites the expectation via \textbf{Lemma~\ref{lemma:expectation_equality}}. Based on the expanded terms, the proof concludes:
\begin{align*}
\sup_{p_\ell}\E_{q_{\tilde{\alpha}}(\vy_0,\vy_t)}&[\log p_\ell(\vy_0\mid \vy_t)]-\sup_{p}\E_{q_\alpha(\vx_0,\vx_t)}[\log p(\vx_0\mid \vx_t)] \\
&= \E_{q_{\tilde{\alpha}}(\vy_0,\vy_t)}[\log q_\text{data\_y}(\vy_0\mid \vy_t)]-\E_{q_{\tilde{\alpha}}(\vy_0,\vy_t)}[\log q_\text{data}(\f_\ell^{-1}(\vy_0) \mid \g_\ell(\vy_t))] \\
&= \E_{q_{\tilde{\alpha}}(\vy_0,\vy_t)}\left[\log\frac{ q_\text{data\_y}(\vy_0\mid \vy_t)}{q_\text{data}(\f_\ell^{-1}(\vy_0) \mid \g_\ell(\vy_t))}\right] \\
&= \E_{q_{\tilde{\alpha}}(\vy_t)}\left[\D_\text{KL}(q_\text{data\_y}(\vy_0\mid \vy_t)\,\| \,q_\text{data}\circ \f_\ell^{-1}(\vy_0 \mid \g_\ell(\vy_t))\right] \\
&= \E_{q_{\tilde{\alpha}}(\vy_t)}\left[\D_\text{KL}(q_\text{data\_y}(\vy_0\mid \vy_t)\,\| \,q_\text{data\_y}(\vy_0 \mid \g_\ell(\vy_t))\right] \geq 0.
\end{align*}
Rearranging the inequality yields $\sup_{p_{\ell}} \E_{q_{\tilde{\alpha}}(\vy_0,\vy_t)}\!\big[\log p_\ell(\vy_0| \vy_t)\big]\geq \sup_{p}\E_{q_{\alpha}(\vx_0,\vx_t)}\!\big[\log  p(\vx_0| \vx_t)\big]$ and the equality holds when $\E_{q_{\tilde{\alpha}}(\vy_t)}\left[\D_\text{KL}(q_\text{data\_y}(\vy_0\mid \vy_t)\,\| \,q_\text{data\_y}(\vy_0 \mid \g_\ell(\vy_t))\right] = 0$.
\end{proof}

\begin{lemma}
\label{lemma:invariant_schedule}
(cf.~\cite{sahoo2024simplifieddiff, shi2024simplifieddiff, kingma2023variationaldiffusionmodels}) Let $\alpha_t,\tilde{\alpha}_t:[0,1]\to [1,0]$ be two strictly decreasing, continuous bijections with $\alpha_0=\tilde{\alpha}_0=1$ and $\alpha_1=\tilde{\alpha}_1=0$. For any function $F:(0,1)\to \R$ integrable on $(0,1)$, define:
\begin{equation}
    I(\alpha)\triangleq \int_{0}^1 \frac{\alpha'_t}{1-\alpha_t} F(\alpha_t)dt.
\end{equation}
Then, we have:
\begin{equation}
    I(\alpha)=I(\tilde{\alpha}).
\end{equation}
\begin{proof}
Let $s=\alpha_t$. Then, $ds=\alpha'_tdt$. As $t:0\to 1$, we have $s:1\to 0$. Hence,
\begin{align*}
    I(\alpha)=\int_{t=0}^1 \frac{\alpha'_t}{1-\alpha_t} F(\alpha_t)dt
    &=\int_{s=1}^0 \frac{1}{1-s} F(s)ds.
\end{align*}
Let $s=\tilde{\alpha}_t$. Then, $ds=\tilde{\alpha}'_tdt$. As $s:1\to 0$, we have $t:0\to 1$. Hence,
\begin{align*}
    \int_{s=1}^0 \frac{1}{1-s} F(s)ds=\int_{t=0}^1 \frac{\tilde{\alpha}'_t}{1-\tilde{\alpha}_t} F(\tilde{\alpha}_t)dt=I(\tilde{\alpha}).
\end{align*}
Therefore, $I(\alpha)=I(\tilde{\alpha})$.
\end{proof}
\end{lemma}

\begin{proposition}
\label{proposition:x_y_bound}
Given a scheduling function $\alpha_t$, the following inequality holds:
\begin{equation}
\label{eq:yx_bound_apx}
\begin{aligned}
\inf_{p_{\ell}}\int_{0}^1 w(t) \E_{q_\alpha(\vy_0,\vy_t)}\!\big[-\log p_{\ell}(\vy_0\mid \vy_t)\big]dt
&\leq \inf_{p} \int_{0}^1 w(t) \E_{q_\alpha(\vx_0,\vx_t)}\!\big[-\log p(\vx_0\mid \vx_t)\big]dt.
\end{aligned}
\end{equation}
The equality holds if and only if:
\begin{equation}
\begin{aligned}
\E_{q_{\tilde{\alpha}}(\vy_t)}\left[\D_\text{KL}(q_{\rm{data\_y}}(\vy_0\mid \vy_t)\,\| \,q_{\rm{data\_y}}(\vy_0 \mid \g_\ell(\vy_t))\right] = 0.
\end{aligned}
\end{equation}
\end{proposition}

\begin{proof}
\begin{equation*}
\begin{aligned}
\inf_{p_{\ell}}\int_{0}^1 w(t) \E_{q_\alpha(\vy_0,\vy_t)}\!\big[-\log p_{\ell}(\vy_0\mid \vy_t)\big]dt &\stackrel{(i)}{=} \inf_{p_{\ell}}\int_{0}^1 \frac{\alpha'_t}{1-\alpha_t} \E_{q_\alpha(\vy_0,\vy_t)}\!\big[\log p_{\ell}(\vy_0\mid \vy_t)\big]dt \\
&\stackrel{(ii)}{=}\int_{0}^1 \frac{\alpha'_t}{1-\alpha_t} \sup_{p_{\ell}} \E_{q_\alpha(\vy_0,\vy_t)}\!\big[\log p_{\ell}(\vy_0\mid \vy_t)\big]dt\\
&\stackrel{(iii)}{=}\int_{0}^1 \frac{\tilde{\alpha}'_t}{1-\tilde{\alpha}_t} \sup_{p_{\ell}} \E_{q_{\tilde{\alpha}}(\vy_0,\vy_t)}\!\big[\log p_{\ell}(\vy_0\mid \vy_t)\big]dt \\
&\stackrel{(iv)}{\leq} \int_{0}^1 \frac{\alpha'_t}{1-\alpha_t} \sup_{p} \E_{q_\alpha(\vx_0,\vx_t)}\!\big[\log p(\vx_0\mid \vx_t)\big]dt\\
&= \inf_{p} \int_{0}^1 \frac{\alpha'_t}{1-\alpha_t}  \E_{q_\alpha(\vx_0,\vx_t)}\!\big[\log p(\vx_0\mid \vx_t)\big]dt\\
&= \inf_{p} \int_{0}^1 w(t)  \E_{q_\alpha(\vx_0,\vx_t)}\!\big[-\log p(\vx_0\mid \vx_t)\big]dt,
\end{aligned}
\end{equation*}
where $(i)$ is because $w(t)=\frac{\alpha'_t}{\alpha_t-1}$, $(ii)$ is due to $\frac{\alpha'_t}{1-\alpha_t} < 0$, $(iii)$ is due to
\textbf{Lemma~\ref{lemma:invariant_schedule}}, and $(iv)$ is due to \textbf{Lemma~\ref{lemma:log_likelihood_greater}}. The equality holds if and only if $\E_{q_{\tilde{\alpha}}(\vy_t)}\left[\D_\text{KL}(q_{\rm{data\_y}}(\vy_0\mid \vy_t)\,\| \,q_{\rm{data\_y}}(\vy_0 \mid \g_\ell(\vy_t))\right] = 0$ according to \textbf{Lemma~\ref{lemma:log_likelihood_greater}}.
\end{proof}

\textbf{Lemmas \ref{lemma:expectation_equality}-\ref{lemma:invariant_schedule}} and \textbf{Proposition~\ref{proposition:x_y_bound}} show that if the scheduling $\tilde{\alpha}_t$ and the corresponding mapping $\g_\ell$ can be identified, the inequality in Eq.~(\ref{eq:yx_bound_apx}) can be successfully established. We follow the same rationale to examine an additional extended theory: as the token granularity $\ell$ grows, the loss value becomes smaller. \textbf{Proposition~\ref{proposition:y_y_bound}} generalizes \textbf{Proposition~\ref{proposition:x_y_bound}} to compare the optimal losses of two MDM-Prime models with different token granularities. 

Let $\ell_1$ and $\ell_2$ denote token granularities. We define $\vy^{(\ell_1)}_0 \triangleq \f_{\ell_1}(\vx_0) \in \mathcal{Y}_1^{L\times \ell_1}$ and $\vy^{(\ell_2)}_0 \triangleq \f_{\ell_2}(\vx_0) \in \mathcal{Y}_2^{L\times \ell_2}$, where the target spaces are defined as $\mathcal{Y}_k = \{0, \dots, \lceil\sqrt[\ell_k]{V}\rceil-1\}$ for $k \in \{1,2\}$. For notational simplicity, we represent the distribution of $\vy^{(\ell_2)}_0$ as $q_{\text{data\_y}}(\vy^{(\ell_2)}_0)$ without explicitly defining a separate distribution symbol for $\vy^{(\ell_1)}_0$. 

We further specify the mappings $\f_{\tilde{\ell}}$ and $\g_{\tilde{\ell}}$ (Eq.~(\ref{eq:apx_coarsening})) for transitioning between granularities with ratio $\tilde{\ell} = \ell_2/\ell_1$. First, $\f_{\tilde{\ell}}: \mathcal{Y}_1^{L\times\ell_1}\to\mathcal{Y}_2^{L\times\ell_2}$ is an invertible function defined by the element-wise operation $f_{\tilde{\ell}}:\mathcal{Y}_1\to \mathcal{Y}_2^{\tilde{\ell}}$. We denote the mapping of a sequence as $\vy^{(\ell_2)}_0=\f_{\tilde{\ell}}(\vy^{(\ell_1)}_0)$. With this definition, $\vy^{(\ell_2)}_0$ can be expanded as follows:
\begin{equation*}
\begin{aligned}
\vy^{(\ell_2)}_0=\f_{\tilde{\ell}}(\vy^{(\ell_1)}_0)
&=\left[f_{\tilde{\ell}}(\vy^{(\ell_1), i,1}_0),\cdots,f_{\tilde{\ell}}(\vy^{(\ell_1), i,\ell_1}_0)\right]\\
&=[\vy^{(\ell_2), i,1}_0, \cdots,\vy^{(\ell_2), i,\ell_1}_0]\\
&=[(y^{(\ell_2), i,1}_0, \cdots, y^{(\ell_2), i,\tilde{\ell}}_0), \cdots, (y^{(\ell_2), i,\tilde{\ell}(\ell_1-1)}_0, \cdots,y^{(\ell_2), i,\ell_2}_0)].
\end{aligned}
\end{equation*}
Second, the function $\g_{\tilde{\ell}}$ is defined by the element-wise operation $g_{\tilde{\ell}}$. This function applies the inverse mapping $f^{-1}_{\tilde{\ell}}$ only if all corresponding sub-tokens are unmasked. Otherwise, it returns the masked token $\mathtt{m}$. The definition aligns with Eq.~(\ref{eq:apx_coarsening}) as follows:
\begin{equation*}
    g_{\frac{\ell_2}{\ell_1}}(\vy^{(\ell_2), i,j}_t)=
    \begin{cases}
      f^{-1}_{\frac{\ell_2}{\ell_1}}(\vy^{(\ell_2),i,j}_t), & \text{if }  y^{i,\frac{\ell_2}{\ell_1}(j-1)+k}_t\neq \texttt{m}, \forall k\in\{1,...,\frac{\ell_2}{\ell_1}\},\\
      \texttt{m}, & \text{otherwise}.
    \end{cases}
\end{equation*}

\begin{proposition}
\label{proposition:y_y_bound}
Let $\ell_1,\ell_2$ be token granularities satisfying $1 < \ell_1 < \ell_2$ and $\frac{\ell_2}{\ell_1}\in\N$. Then, the optimal loss is monotone with respect to the token granularity:
\begin{equation}
\begin{aligned}
\inf_{p_{\ell_2}}\int_{0}^1 w(t) \E_{q_\alpha(\vy^{(\ell_2)}_0,\vy_t^{(\ell_2)})}\!\big[-\log p_{\ell_2}(\vy^{(\ell_2)}_0\mid \vy^{(\ell_2)}_t)\big]dt
&\leq \inf_{p_{\ell_1}} \int_{0}^1 w(t) \E_{q_\alpha(\vy_0^{(\ell_1)},\vy_t^{(\ell_1)})}\!\big[-\log p_{\ell_1}(\vy_0^{(\ell_1)}\mid \vy_t^{(\ell_1)})\big]dt.
\end{aligned}
\end{equation}
The equality holds if and only if:
\begin{equation}
\begin{aligned}
\E_{q_{\tilde{\alpha}}(\vy_t^{(\ell_2)})}\left[\D_\text{KL}(q_{\rm{data\_y}}(\vy_0^{(\ell_2)}\mid \vy_t^{(\ell_2)})\,\| \,q_{\rm{data\_y}}(\vy_0^{(\ell_2)} \mid \g_{\frac{\ell_2}{\ell_1}}(\vy_t^{(\ell_2)}))\right] = 0.
\end{aligned}
\end{equation}
\end{proposition}

\begin{proof}
The proof first identifies the transformation between $\vy^{(\ell_1)}_0$ and $\vy^{(\ell_2)}_0$. Then, it establishes the inequality by following the proof presented in \textbf{Proposition~\ref{proposition:x_y_bound}}. 

Given that $\tilde{\ell}=\frac{\ell_2}{\ell_1}$, the invertible transformation $\f_{\ell_2}$ can be decomposed as $\f_{\ell_2} = \f_{\tilde{\ell}} \circ \f_{\ell_1}$, and $\vy^{(\ell_2)}_0$ can be derived through $\vy^{(\ell_1)}_0$ according to $\vy^{(\ell_2)}_0=\f_{\tilde{\ell}}(\vy^{(\ell_1)}_0)$. In addition, let $\tilde{\alpha}_t=\alpha_t^{1/\tilde{\ell}}$. We can establish analogues to \textbf{Lemma~\ref{lemma:expectation_equality}} and \textbf{Lemma~\ref{lemma:log_likelihood_greater}} using the same derivation steps:
\begin{equation*}
    \E_{q_\alpha(\vy_0^{(\ell_1)}, \vy_t^{(\ell_1)})}[F(\vy_0^{(\ell_1)}, \vy_t^{(\ell_1)})] = \E_{q_{\tilde{\alpha}}(\vy_0^{(\ell_2)}, \vy_t^{(\ell_2)})}[F(\f^{-1}_{\tilde{\ell}}(\vy_0^{(\ell_2)}), \g_{\tilde{\ell}}(\vy_t^{(\ell_2)}))].
\end{equation*}
\begin{equation*}
    \sup_{p_{\ell_1}}\E_{q_\alpha(\vy_0^{(\ell_1)}, \vy_t^{(\ell_1)})}\!\big[\log  p_{\ell_1}(\vy_0^{(\ell_1)}\mid \vy_t^{(\ell_1)})\big]
 \le \sup_{p_{\ell_2}} \E_{q_{\tilde{\alpha}}(\vy_0^{(\ell_2)}, \vy_t^{(\ell_2)})} \!\big[\log p_{\ell_2}(\vy_0^{(\ell_2)}\mid \vy_t^{(\ell_2)})\big],\quad \forall t\in[0,1].
\end{equation*}

Finally, following the derivation of \textbf{Proposition~\ref{proposition:x_y_bound}}, we have:
\begin{equation*}
\begin{aligned}
\inf_{p_{\ell_2}}\int_{0}^1 & w(t) \E_{q_\alpha(\vy^{(\ell_2)}_0,\vy_t^{(\ell_2)})}\!\big[-\log p_{\ell_2}(\vy^{(\ell_2)}_0\mid \vy^{(\ell_2)}_t)\big]dt \\
&=\inf_{p_{\ell_2}} \int_{0}^1 \frac{\alpha'_t}{1-\alpha_t}
\E_{q_\alpha(\vy_0^{(\ell_2)},\vy_t^{(\ell_2)})}\!\big[\log p_{\ell_2}(\vy_0^{(\ell_2)}\mid \vy_t^{(\ell_2)})\big]dt \\
&=\int_{0}^1 \frac{\alpha'_t}{1-\alpha_t} \sup_{p_{\ell_2}} \E_{q_\alpha(\vy_0^{(\ell_2)},\vy_t^{(\ell_2)})}\!\big[\log p_{\ell_2}(\vy_0^{(\ell_2)}\mid \vy_t^{(\ell_2)})\big]dt\\
&=\int_{0}^1 \frac{\tilde{\alpha}'_t}{1-\tilde{\alpha}_t} \sup_{p_{\ell_2}} \E_{q_{\tilde{\alpha}}(\vy_0^{(\ell_2)},\vy_t^{(\ell_2)})}\!\big[\log p_{\ell_2}(\vy_0^{(\ell_2)}\mid \vy_t^{(\ell_2)})\big]dt \\
&\leq \int_{0}^1 \frac{\alpha'_t}{1-\alpha_t} \sup_{p_{\ell_1}} \E_{q_{\tilde{\alpha}}(\vy_0^{(\ell_1)}, \vy_t^{(\ell_1)})} \!\big[\log p_{\ell_1}(\vy_0^{(\ell_1)}\mid \vy_t^{(\ell_1)})\big]dt\\
&= \inf_{p_{\ell_1}}\int_{0}^1 \frac{\alpha'_t}{1-\alpha_t} \E_{q_\alpha(\vy^{(\ell_1)}_0,\vy_t^{(\ell_1)})}\!\big[\log p_{\ell_1}(\vy^{(\ell_1)}_0\mid \vy^{(\ell_1)}_t)\big]dt\\
&= \inf_{p_{\ell_1}} \int_{0}^1 w(t) \E_{q_\alpha(\vy_0^{(\ell_1)},\vy_t^{(\ell_1)})}\!\big[-\log p_{\ell_1}(\vy_0^{(\ell_1)}\mid \vy_t^{(\ell_1)})\big]dt.
\end{aligned}
\end{equation*}
The equality holds if and only if $\E_{q_{\tilde{\alpha}}(\vy_t^{(\ell_2)})}\left[\D_\text{KL}(q_{\rm{data\_y}}(\vy_0^{(\ell_2)}\mid \vy_t^{(\ell_2)})\,\| \,q_{\rm{data\_y}}(\vy_0^{(\ell_2)} \mid \g_{\tilde{\ell}}(\vy_t^{(\ell_2)}))\right] = 0$.
\end{proof}

\subsubsection{Proofs of Propositions 3.1 and 3.2}
\label{sec:apx:maximum_ent_subtokenizer}
As discussed in Section~\ref{sec:methodology:entropy}, the design of the subtokenizer $\f_\ell$ is critical for reducing the training loss. In this section, we provide the theoretical justifications for \textbf{Propositions~\ref{proposition:decomposition}} and \textbf{\ref{proposition:entropy_max}}. We begin by deriving \textbf{Lemmas~\ref{lemma:data_entropy}-\ref{lemma:kernel_entropy}}, which lead to the formulation of \textbf{Proposition~\ref{proposition:decomposition}}. 
\begin{lemma}
\label{lemma:data_entropy}
The following equality holds:
\begin{equation}
\E_{q_{\rm{data\_y}}(\vy_0)}[\log q_{\rm{data\_y}}(\vy_0)]=\E_{q_{\rm{data}}(\vx_0)}[\log q_{\rm{data}}(\vx_0)].
\end{equation}
\end{lemma}
\begin{proof}
\begin{equation*}
\begin{aligned}
\E_{q_\text{data}(\vx_0)}[\log q_\text{data}(\vx_0)]&=\sum_{\vx_0} q_\text{data}(\vx_0) \log q_\text{data}(\vx_0) \\
&=\sum_{\vy_0} q_\text{data}\circ \f^{-1}_\ell(\vy_0) \log q_\text{data}\circ \f^{-1}_\ell(\vy_0) \\
&=\E_{q_\text{data\_y}(\vy_0)}[\log q_\text{data\_y}(\vy_0)].
\end{aligned}
\end{equation*}
\end{proof}

\begin{lemma}
\label{lemma:kernel_entropy}
The following equality holds:
\begin{equation}
\E_{q_\alpha(\vy_0,\vy_t)}[\log q_\alpha(\vy_t\mid \vy_0)]=L\ell \left( (1-\alpha_t) \log (1-\alpha_t) + \alpha_t \log \alpha_t \right)
\end{equation}
\end{lemma}
\begin{proof}
\begin{equation*}
\begin{aligned}
\E_{q_\alpha(\vy_0,\vy_t)}[\log q_\alpha(\vy_t\mid \vy_0)]&= \E_{q_\alpha(\vy_0,\vy_t)}\left[\log \left(\prod_{i=1}^L \prod_{j=1}^{\ell} q_\alpha(y_t^{i,j}\mid y_0^{i,j})\right)\right]\\
&= \sum_{i=1}^L \sum_{j=1}^{\ell}  \E_{q_\alpha(\vy_0,\vy_t)}\left[\log q_\alpha(y_t^{i,j}\mid y_0^{i,j})\right]\\
&= \sum_{i=1}^L \sum_{j=1}^{\ell} \E_{q_\alpha(\vy_0,\vy_t)}\left[\log \left((1-\alpha_t)\delta_{\texttt{m}}(y^{i,j}_t)+\alpha_t \delta_{y_0^{i,j}}(y^{i,j}_t)\right)\right]\\
&\stackrel{(i)}{=} \sum_{i=1}^L \sum_{j=1}^{\ell} \left( (1-\alpha_t) \log (1-\alpha_t) + \alpha_t \log \alpha_t \right)\\
&= L\ell \left( (1-\alpha_t) \log (1-\alpha_t) + \alpha_t \log \alpha_t \right),
\end{aligned}
\end{equation*}
where $(i)$ is because $y^{i,j}_t=\m$ with probability $1-\alpha_t$ while $y^{i,j}_t=y^{i,j}_0$ with probability $\alpha_t$ (cf. \textbf{Proposition A2} in~\cite{chao2025mdmprime}).
\end{proof}

\textbf{Proposition 3.1.}~The objective (Eq.~(\ref{eq:prime_elbo})) can be decomposed into an $\f_\ell$-independent and an $\f_\ell$-dependent term as follows:
\begin{equation}
\begin{aligned}
\inf_{\f_\ell} \inf_{p_\ell} \mathcal{L}^{(\ell)}  = \int_{0}^1 w(t) \Big(\underbrace{\mathcal{H}(\vy_0, \vy_t)}_\text{Independent on $\f_\ell$.}-\sup_{\f_\ell} \underbrace{\mathcal{H}(\vy_t)}_\text{Dependent on $\f_\ell$.} \Big)dt,
\end{aligned}
\end{equation}
where $\mathcal{H}(\vy_t)\triangleq \E_{q_\alpha(\vy_t)}[-\log q_\alpha(\vy_t)]$ and $\mathcal{H}(\vy_0, \vy_t)\triangleq \E_{q_\alpha(\vy_0,\vy_t)}[-\log q_\alpha(\vy_0,\vy_t)]$ represent the entropy of $\vy_t$ and the joint entropy of $\vy_0,\vy_t$, respectively.

\begin{proof}
\begin{equation*}
\begin{aligned}
&\inf_{\f_\ell} \inf_{p_\ell} \mathcal{L}^{(\ell)}  \\
&= \inf_{\f_\ell} \inf_{p_\ell} \int_{0}^1 w(t) \E_{q_\alpha(\vy_0,\vy_t)}\!\big[-\log p_\ell(\vy_0\mid \vy_t)\big]dt \\
&\stackrel{(i)}{=}\inf_{\f_\ell}\int_{0}^1 w(t) \E_{q_\alpha(\vy_0,\vy_t)}\!\big[-\log q_\text{data\_y}(\vy_0\mid \vy_t)\big]dt \\
&= \int_{0}^1 -w(t) \sup_{\f_\ell} \E_{q_\alpha(\vy_0,\vy_t)}\!\big[\log q_\text{data\_y}(\vy_0\mid \vy_t)\big]dt \\
&= \int_{0}^1 -w(t) \sup_{\f_\ell}\left(\E_{q_\alpha(\vy_0,\vy_t)}\!\big[\log q_{\text{data\_y}}(\vy_0) + \log q_\alpha(\vy_t\mid \vy_0) - \log q_\alpha(\vy_t)\big]\right)dt \\
&= \int_{0}^1 -w(t) \left(\underbrace{\E_{q_{\text{data\_y}}(\vy_0)}[\log q_{\text{data\_y}}(\vy_0)]}_\text{$(1)$ Independent on $\f_\ell$.} + \underbrace{\E_{q_\alpha(\vy_0,\vy_t)}[\log q_\alpha(\vy_t\mid \vy_0)]}_\text{$(2)$ Independent on $\f_\ell$.} + \sup_{\f_\ell}\underbrace{\E_{q_\alpha(\vy_t)}[-\log q_\alpha(\vy_t)]}_\text{$(3)$ dependent on $\f_\ell$.} \right)dt \\
&=\int_{0}^1 -w(t) \left(\E_{q_\alpha(\vy_0,\vy_t)}[\log q_\alpha(\vy_0,\vy_t)] + \sup_{\f_\ell} \E_{q_\alpha(\vy_t)}[-\log q_\alpha(\vy_t)] \right)dt, \\
&=\int_{0}^1 w(t) \Big(\underbrace{\mathcal{H}(\vy_0, \vy_t)}_\text{Independent on $\f_\ell$.}-\sup_{\f_\ell} \underbrace{\mathcal{H}(\vy_t)}_\text{Dependent on $\f_\ell$.} \Big)dt,
\end{aligned}
\end{equation*}
where $(i)$ is derived in \textbf{Lemma~\ref{lemma:log_likelihood_greater}}. Term $(1)$ corresponds to the negative data entropy, $-\mathcal{H}(\vx_0) \triangleq \mathbb{E}_{q_\text{data}(\vx_0)}[\log q_\text{data}(\vx_0)]$, which is independent of $\f_\ell$ according to \textbf{Lemma~\ref{lemma:data_entropy}}. Term $(2)$ represents the negative conditional entropy, $-\mathcal{H}(\vy_t|\vy_0)$, which equals to the constant $-L\ell [ (1-\alpha_t) \log (1-\alpha_t) + \alpha_t \log \alpha_t ]$ by \textbf{Lemma~\ref{lemma:kernel_entropy}}. Finally, term $(3)$ is $\mathcal{H}(\vy_t)$ and is the only component dependent on $\f_{\ell}$, making it the primary target of the optimization objective.
\end{proof}


Next, we show that the entropy of $\vy_t$ is bounded. We first present \textbf{Lemma~\ref{lemma:yt_entropy}} from \cite{austin2021structurediff, chao2025mdmprime}. Then, we provide the proof of \textbf{Proposition~\ref{proposition:entropy_max}}.

\begin{lemma}(cf. Proposition A2 in~\cite{chao2025mdmprime})
\label{lemma:yt_entropy}
\begin{equation}
\mathcal{H}(y^{i,j}_t)=\alpha_t \mathcal{H}(y^{i,j}_0)-(1-\alpha_t) \log (1-\alpha_t) - \alpha_t \log \alpha_t.
\end{equation}
\end{lemma}
\begin{proof}
Please see Eq.~(A12) in \cite{chao2025mdmprime} and Section~A7 of \cite{austin2021structurediff}.
\end{proof}

\textbf{Proposition 3.2.}~The entropy of $\vy_t$ is bounded:
\begin{equation}
\mathcal{H}(\vy_t) \leq L\ell (h(\alpha_t)+\alpha_t \log b),
\end{equation}
where $h(\alpha_t)= -(1-\alpha_t) \log (1-\alpha_t) - \alpha_t \log \alpha_t$. The equality holds if and only if each unmasked $y_t^{i,j}$ is uniformly distributed on $\Y$.
\begin{proof}
\begin{equation*}
\begin{aligned}
\mathcal{H}(\vy_t) &\stackrel{(i)}{\leq} \sum_{i=1}^L \sum_{j=1}^{\ell} \mathcal{H}(y^{i,j}_t) \\
&\stackrel{(ii)}{=} \sum_{i=1}^L \sum_{j=1}^{\ell} \left(\alpha_t \mathcal{H}(y^{i,j}_0)-(1-\alpha_t) \log (1-\alpha_t) - \alpha_t \log \alpha_t \right)\\
&\stackrel{(iii)}{\leq} \sum_{i=1}^L \sum_{j=1}^{\ell} \left(\alpha_t \log b-(1-\alpha_t) \log (1-\alpha_t) - \alpha_t \log \alpha_t \right) \\
&=L\ell\left(\alpha_t \log b-(1-\alpha_t) \log (1-\alpha_t) - \alpha_t \log \alpha_t \right) \\
&=L\ell\left(\alpha_t \log b+h(\alpha_t) \right),
\end{aligned}
\end{equation*}
where $(i)$ is due to the subadditivity of entropy, $(ii)$ is due to \textbf{Lemma~\ref{lemma:yt_entropy}}, and $(iii)$ is because $|\Y|=b$. The equality holds when each sub-token is independent and uniformly distributed.
\end{proof}

\subsubsection{Estimating the Variational Bound of MDM-Prime}
\label{sec:apx:proof:nll}
The training objective of MDM-Prime (Eq.~(\ref{eq:prime_elbo})) requires an additional marginalization step to constitute a valid variational upper bound on the negative log likelihood (NLL). As the bound derivation assumes mutually independent sub-token predictions, the joint parameterization $p_\theta(\vy^i_0|\vy_t)$ in~\cite{chao2025mdmprime}, which introduces dependencies among sub-tokens within each token group, must first be marginalized. Specifically, let $\mathcal{M}_j\triangleq \{y^{i,k}_0 \,|\, k\neq j,\,\, y^{i,k}_0\in \Y\}$; the marginalized distribution is:
\begin{equation}
p_\theta(y^{i,j}_0|\boldsymbol{y}_t)\triangleq\sum_{\vy^{i,\,\neq j}_0 \in \mathcal{M}_j}p_\theta(y^{i,j}_0,\vy_0^{i,\,\neq j}|\vy_t).
\end{equation}
Substituting the fully factorized form $p_\theta(\vy_0|\vy_t)=\prod_{i=1}^L\prod_{j=1}^\ell p_\theta(y^{i,j}_0|\vy_t)$ into Eq.~(\ref{eq:prime_elbo}) then yields a valid variational upper bound on the NLL. This marginalization can be implemented efficiently via filtering on the model's logit outputs.
\subsection{Technical Specifications}
\label{sec:apx:specification}
This section provides further implementation details for the proposed framework. Section~\ref{sec:apx:specification:architecture} describes the architectures of MDM, MDM-Prime, and MDM-Prime-v2, while Section~\ref{sec:apx:specification:subtokenizer} details a practical implementation of the subtokenizer $\f_\ell$.

\subsubsection{Model Architecture of MDM-Prime and MDM-Prime-v2}
\label{sec:apx:specification:architecture}
Fig.~\ref{fig:architecture}~(a) compares the architectures of MDM, MDM-Prime, and MDM-Prime-v2. MDM-Prime and MDM-Prime-v2 share exactly the same architecture. As detailed in Section~\ref{sec:background:prime}, transitioning from MDM to MDM-Prime (or MDM-Prime-v2) requires only a simple modification of the embedding table~\cite{chao2025mdmprime}. In this setup, sub-token embeddings are aggregated into token embeddings, and the neural network (i.e., Transformer $(\theta)$ in the figure) subsequently processes at the token level. The decoder maintains its original dimensionality to effectively model intra-token dependencies. Specifically, let $H$ represent the hidden layer embedding size of the transformer. The transformers in both architectures accept inputs of size $\R^{H\times L}$ and outputs logits with $\R^{C\times L}$. This design maintains the core model architecture and ensures that the computational cost (FLOPs) per forward pass remains unchanged.

\begin{figure*}[t!]
    \centering
    \vspace{-0.5em}
    \includegraphics[width=\linewidth]{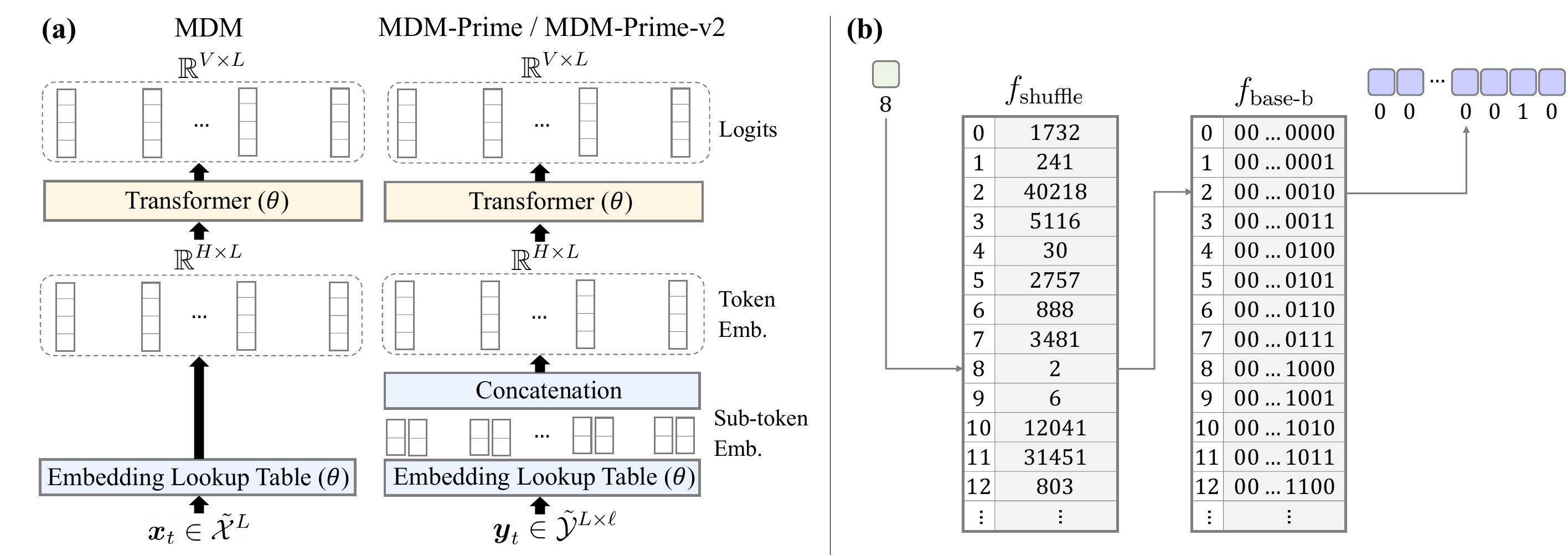}
    \vspace{-1em}
    \caption{(a) Comparison of MDM, MDM-Prime, and MDM-Prime-v2 architectures. $V$ denotes the vocabulary size, $L$ is the sequence length, $\ell$ is the token granularity, and $H$ is the Transformer hidden embedding dimension. (b) Implementation of the subtokenizer $f_\ell= f_\text{base-$b$} \circ f_\text{shuffle}$. The process utilizes two lookup tables to map indices efficiently; the inverse operation $f_\ell^{-1}$ is performed by reversing the lookup sequence. }
    \vspace{-1em}
    \label{fig:architecture}
\end{figure*}

\subsubsection{Subtokenizer Implementation}
\label{sec:apx:specification:subtokenizer}
Fig.~\ref{fig:architecture}~(b) depicts how the element-wise operation $f_\ell= f_\text{base-$b$} \circ f_\text{shuffle}$ of subtokenizer $\f_\ell$ is implemented in practice. This element-wise operation relies on lookup tables: to transform a token $x_0^i$ into a sub-token sequence $\vy_0^i$, the function sequentially retrieves the shuffled index and its corresponding base-$b$ representation. Similarly, the inverse operation $f^{-1}_\ell$ is executed via a reverse lookup. Both directions can be efficiently implemented using dictionaries in Python.

Please note that the $\f_\ell$ is static over training, and the lookup tables for index shuffling and base-$b$ encoding are also fixed prior to training. In Section \ref{sec:apx:experiments}, we show that the performance of MDM-Prime-v2 is robust to the choice of the random seed used for the index shuffling operation.

We do not recommend parameterizing and iteratively updating the subtokenizer during training for two reasons. First, our analysis in Section~\ref{sec:methodology:entropy} suggests that a non-learnable index shuffling operation is sufficiently effective and the sub-token entropy is near the theoretical maximum. Introducing an additional optimization phase for the subtokenizer is likely to yield only marginal gains while significantly increasing computational overhead. Second, updating $\f_\ell$ during training would induce a continuous shift in both the target distribution $q_{\text{data\_y}}(\vy_0)$ and the latent distribution $q_\alpha(\vy_t)$. This non-stationarity in the learning objective can result in severe training instability.

\begin{figure}[t]
    \centering
    \vspace{-0.25em}
    \includegraphics[width=\linewidth]{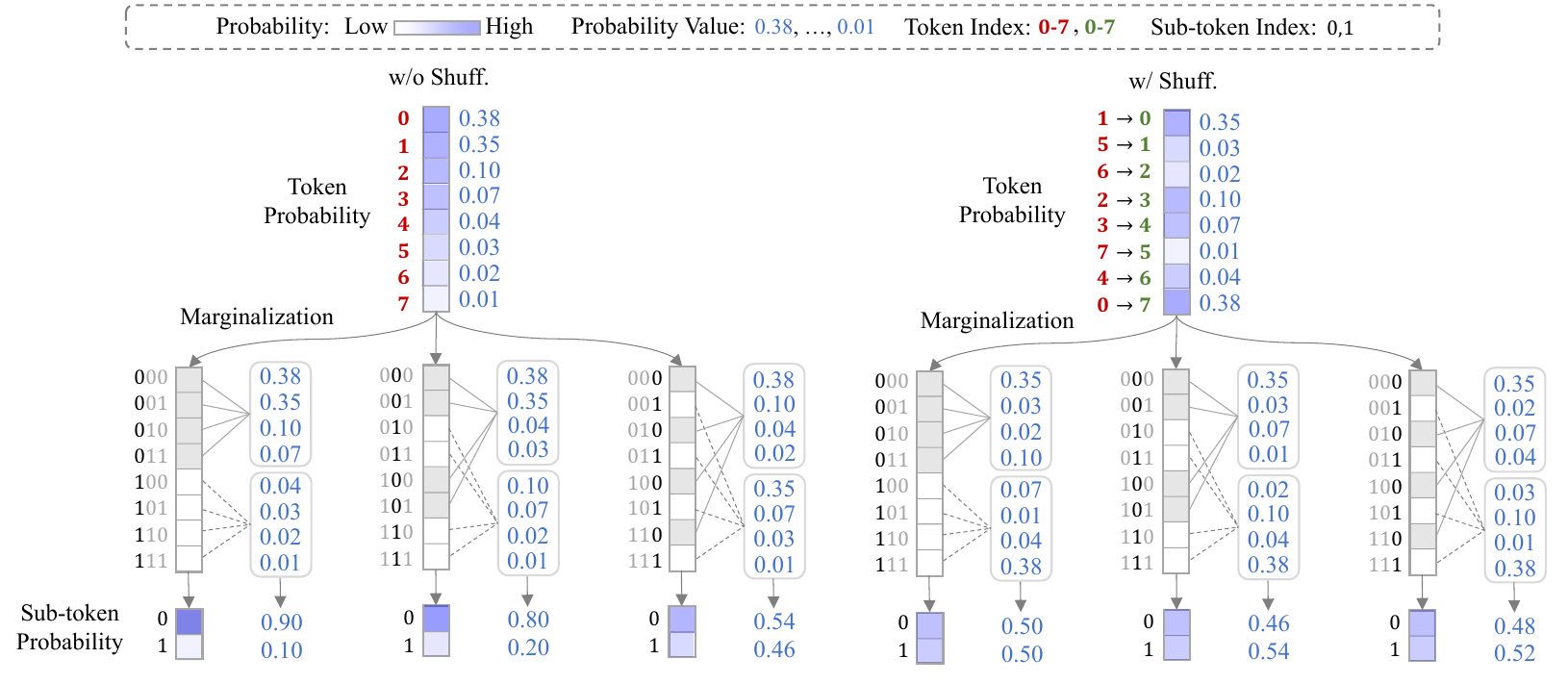}
    \vspace{-1.75em}
    \caption{A detailed numerical example of Fig.~\ref{fig:entropy}~(a). In this plot, $V=8$, $\ell=3$ and $b=2$. Color intensity represents probability values, where darker blue denotes higher probability. Marginalization is performed by summing entries corresponding to sub-token indices 0 (gray) and 1 (white), respectively. Sub-tokens encoded from standard indices with a structured distribution exhibit low entropy, while the shuffle operation results in sub-tokens with higher entropy.}
    \label{fig:apx:shuffle_example_numerical}
    \vspace{-0.75em}
\end{figure}

\subsection{Experimental Configurations}
\label{sec:apx:configuration}
This section provides supplemental experimental details and hyperparameter configurations. The settings for the experiments described in Sections~\ref{sec:experiment:scaling} and \ref{sec:experiment:1B_model} are detailed in Sections \ref{sec:apx:configuration:scaling} and \ref{sec:apx:configuration:1b}, respectively. 

\subsubsection{Experiments in Section 4.1}
\label{sec:apx:configuration:scaling}

\textbf{Dataset and Preprocessing Method.}~The training is performed on the C4 dataset~\cite{raffel2023exploringlimitstransferlearning}. Text data is tokenized using the GPT-2 tokenizer~\cite{radford2019language}. The maximum sequence length $L$ and the vocabulary size $V$ are defined as 2,048 and 50,257, respectively. The documents are concatenated with the \texttt{<eod>} token as a separator and then wrapped into samples of length 2,048.

\begin{table*}[b!]
\begin{minipage}{0.52\linewidth}
\renewcommand{\arraystretch}{1.08}
\newcommand{\boldtoprule}{\toprule[1.2pt]}
\newcommand{\boldbottomrule}{\bottomrule[1.2pt]}
\centering
\resizebox{\linewidth}{!} {%
\begin{tabular}{cccc|cccc}
\boldtoprule
$N$ & $H$ & $B$ & Heads  & $N$ & $H$ & $B$ & Heads  \\
\hline
14M & 448  & 6 & 7 & 252M & 1,024  & 20 & 16 \\
25M & 512  & 8 & 8 & 354M & 1,280 & 18  & 10 \\
36M & 576 & 9 & 9 & 413M & 1,280  & 21 & 10 \\
49M & 640 & 10 & 10 & 571M & 1,408 & 24 & 11  \\
64M & 640 & 13 & 10  & 771M & 1,792  & 20 & 14 \\
79M & 640  & 16 & 10 & 1,107M & 2,048  & 22 & 16 \\
106M & 768  & 15 & 12 & 1,529M & 2,304  & 24 & 18 \\
154M & 896  & 16 & 14 & 2,359M & 2,560 & 30 & 20 \\
201M & 1,024  & 16 & 16 & 3,426M & 2,816  & 36 & 22 \\
\boldbottomrule
\end{tabular}}
\caption{Model architecture configurations. $N$ represents the number of non-embedding parameters, $H$ denotes the hidden dimension size,  $B$ is the number of Transformer blocks, and `Heads' refer to the number of attention heads.}
\label{table:parameters}
\end{minipage}\hfill
\begin{minipage}{0.45\linewidth}
\renewcommand{\arraystretch}{1.2}
\newcommand{\boldtoprule}{\toprule[1.2pt]}
\newcommand{\boldbottomrule}{\bottomrule[1.2pt]}
\centering
\resizebox{\linewidth}{!} {%
\begin{tabular}{cccc}
\boldtoprule
 & ARM & MDM & MDM-Prime-v2 \\
\hline
$\alpha$ & $0.309^{+0.030}_{-0.017}$ & $0.345^{+0.016}_{-0.038}$ & $0.379^{+0.023}_{-0.024}$ \\
$\beta$  & $0.334^{+0.013}_{-0.036}$ & $0.342^{+0.027}_{-0.017}$ & $0.380^{+0.014}_{-0.018}$ \\
$A$      & $217^{+140}_{-49}$        & $355^{+98}_{-158}$        & $502^{+210}_{-156}$ \\
$B$      & $1096^{+304}_{-531}$      & $1969^{+1430}_{-551}$     & $4324^{+1463}_{-1348}$ \\
$E$      & $2.029^{+0.042}_{-0.077}$ & $2.420^{+0.049}_{-0.052}$ & $1.578^{+0.037}_{-0.056}$ \\
\boldbottomrule
\end{tabular}}
\caption{Regression parameters for ARM, MDM, and MDM-Prime-v2 fitted according to Chinchilla's scaling laws~\cite{hoffmann2022chinchilla}. $E$ denotes the irreducible loss (entropy of the natural text distribution). The terms $A\times N^{-\alpha}$ and $B\times D^{-\beta}$ account for errors arising from finite model size and limited training tokens, respectively.}
\label{table:scaling_parameters}
\end{minipage}
\end{table*}

\textbf{Training Configuration.}~We employ a Transformer architecture~\cite{vaswani2017transformer} incorporating Rotary Positional Embeddings (RoPE)~\cite{su2023roformerenhancedtransformerrotary}, SwiGLU activation~\cite{shazeer2020gluvariantsimprovetransformer}, RMSNorm~\cite{zhang2019rootmeansquarelayer}, and QK-normalization~\cite{dehghani2023qknorm}. Model sizes ($N$) range from 14M to 3.4B parameters; detailed specifications for depth, hidden dimensions, and attention heads are summarized in Table~\ref{table:parameters}. We utilize a linear scheduling function $\alpha_t = 1 - t$ for the diffusion process. Optimization is performed using AdamW~\cite{loshchilov2019adamw} with $\beta_1 = 0.9$ and $\beta_2 = 0.999$. The learning rate is set to $2 \times 10^{-4}$ with a cosine decay to a minimum of $2 \times 10^{-5}$. We use a batch size of 256 for $N \leq 1.5$B and 512 for larger models. The coding framework is based on Megatron-LM~\cite{megatron-lm}. For hardware, models with $N \leq 400$M are trained on eight NVIDIA L40 GPUs (48 GB), while larger models ($N > 400$M) are trained on eight NVIDIA H100 GPUs (80 GB).

\subsubsection{Experiments in Section 4.2}
\label{sec:apx:configuration:1b}
\begin{table}[t]
    \centering
    \small 
    \renewcommand{\arraystretch}{1} 
    \newcommand{\boldtoprule}{\toprule[1.2pt]}
    \newcommand{\boldbottomrule}{\bottomrule[1.2pt]}
    \begin{tabular}{lcccccccc} 
        \boldtoprule
                 & SocialIQA & OBQA & McTaco & ARC-e & TruthfulQA & BoolQ & SciQ  & ANLI  \\ 
        \midrule
        Steps    & 1,024  & 512 & 512  & 512 & 256 & 256 & 256 & 128 \\ 
        \boldbottomrule
    \end{tabular}
    \vspace{-0.25em}
    \caption{Number of steps used for zero-shot evaluation of MDM-Prime-v2. SMDM is fixed at 1,024 steps; MDM-Prime-v2 uses no more than 1,024 steps.}
    \label{tab:apx:step_setup}
\end{table}

\textbf{Dataset and Preprocessing Method.}~Following TinyLLaMA~\cite{zhang2024tinyllama} and SMDM~\cite{nie2025scalingmaskeddiffusionmodels}, training is conducted on the SlimPajama dataset \cite{cerebras2023slimpajama}, with text tokenized using the LLaMA tokenizer. The maximum sequence length $L$ is set to 2,048, and the vocabulary size $V$ is 32,000. Additionally, following the approach in \cite{nie2025scalingmaskeddiffusionmodels}, we randomize the sequence length for 1\% of the training data. This is implemented by sampling a length $L'$ from a uniform distribution $\mathcal{U}(1, 2048)$ for each affected sequence.

\textbf{Evaluation Method.}~We utilize the \texttt{lm-evaluation-harness} package~\cite{eval-harness} to measure zero-shot accuracy on eight commonsense reasoning benchmarks: SciQ~\cite{Welbl2017CrowdsourcingMC}, SocialIQA~\cite{sap2019social}, MCTaco~\cite{zhou2019goingvacationtakeslonger}, TruthfulQA~\cite{lin-etal-2022-truthfulqa}, BoolQ~\cite{clark2019boolq}, ANLI~\cite{nie-etal-2020-adversarial}, and ARC-e (Easy)~\cite{Clark2018ThinkYH}. Detailed task descriptions are provided in Table~\ref{tab:apx:commonsense}. For prediction, we concatenate each answer candidate with the question prompt and select the option that yields the highest likelihood under the diffusion model.

To evaluate sample quality, we use generative perplexity (Gen PPL) and entropy as complementary metrics, measuring sample fidelity and diversity respectively. Since Gen PPL can be manipulated through improper temperature selection at the expense of entropy~\cite{zheng2025mdm}, we tune the sampling temperatures and seeds of SMDM and MDM-Prime-v2 to yield comparable entropy values, ensuring a fair Gen PPL comparison. Gen PPL is computed using TinyLLaMA~\cite{zhang2024tinyllama}, trained on 3T tokens.

\begin{table}[t]
    \centering
    \small 
    \renewcommand{\arraystretch}{1.1} 
    \newcommand{\boldtoprule}{\toprule[1.2pt]}
    \newcommand{\boldbottomrule}{\bottomrule[1.2pt]}
    \begin{tabularx}{\linewidth}{lX} 
        \boldtoprule
        \textbf{Benchmark} & \textbf{Description} \\ 
        \midrule
        SciQ~\cite{Welbl2017CrowdsourcingMC} & A collection of crowdsourced multiple-choice science exam questions covering Biology, Chemistry, and Physics. \\
        SocialIQA~\cite{sap2019social} & A large-scale benchmark testing understanding of social interactions, emotional reactions, and motivations. \\
        McTaco~\cite{zhou2019goingvacationtakeslonger} & A dataset focusing on temporal commonsense reasoning, requiring answers about duration, frequency, and event ordering. \\
        TruthfulQA~\cite{lin-etal-2022-truthfulqa} & A benchmark measuring whether models mimic human falsehoods, focusing on misconceptions and superstitions. \\
        BoolQ~\cite{clark2019boolq} & Naturally occurring yes/no questions paired with Wikipedia paragraphs containing the answer. \\
        ANLI~\cite{nie-etal-2020-adversarial} & Adversarial natural language inference dataset collected via a human-and-model-in-the-loop process. \\
        ARC-e~\cite{Clark2018ThinkYH} & The easy subset of the AI2 Reasoning Challenge, containing grade-school science questions. \\
        OBQA~\cite{mihaylov2018suitarmorconductelectricity} & A QA dataset modeled after open book exams, containing elementary science questions that require reasoning with common knowledge. \\
        \boldbottomrule
    \end{tabularx}
    \vspace{-0.25em}
    \caption{An overview of the commonsense reasoning benchmarks used for evaluation.}
    \label{tab:apx:commonsense}
\end{table}

\textbf{Training and Implementation Details.}~We utilize a Transformer architecture \cite{vaswani2017transformer} enhanced with RoPE \cite{su2023roformerenhancedtransformerrotary}, SwiGLU activation \cite{shazeer2020gluvariantsimprovetransformer}, RMSNorm \cite{zhang2019rootmeansquarelayer}, and Grouped Query Attention (GQA) \cite{ainslie2023gqatraininggeneralizedmultiquery}. The model contains 1.028B non-embedding parameters and is trained on 540B tokens. This configuration is identical to our MDM and ARM baselines (i.e., SMDM~\cite{nie2025scalingmaskeddiffusionmodels} and TinyLLaMA~\cite{zhang2024tinyllama}). We adopt the convention of these works and refer to the model as 1.1B. For the diffusion process, we utilize a linear schedule defined by $\alpha_t = 1 - t$.  We set the grouping factor $u$ to $1$ if $t<0.5$ and $\ell$ if $t\geq0.5$ (see Section~\ref{sec:apx:experiments:error_distribution}).  Training is performed using the AdamW optimizer \cite{loshchilov2019adamw} with hyperparameters $\beta_1 = 0.9$ and $\beta_2 = 0.999$. We set the maximum learning rate to $2 \times 10^{-4}$, which follows a cosine decay schedule to a minimum of $2 \times 10^{-5}$. We use a batch size of 256. Our implementation is built upon the lit-gpt framework \cite{lit-gpt}, leveraging FlashAttention \cite{dao2022flashattentionfastmemoryefficientexact} and xFormers \cite{xFormers2022} for enhanced efficiency. The model was trained on sixteen NVIDIA H100 GPUs (80 GB) with a multi-node configuration.

\subsection{Supplementary Experiments}
\label{sec:apx:experiments}
In this section, we present additional experimental results. Section~\ref{sec:apx:experiments:performance_invariant} examines design choices that yield performance invariance. Section~\ref{sec:apx:experiments:error_distribution} provides further analysis of the log-likelihood distribution across time. Section~\ref{sec:apx:experiments:arm_different_l} details the performance of ARM trained under the sub-tokenization scheme.  Section~\ref{sec:apx:experiments:ablation} provides detailed ablation analysis on the proposed two techniques. Section~\ref{sec:apx:experiments:scaling} presents additional scaling analysis.  Finally, Section~\ref{sec:apx:experiments:7B} presents preliminary results of MDM-Prime-v2 at 7B-parameter scale. 


\subsubsection{Performance-Invariant Designs}
\label{sec:apx:experiments:performance_invariant}
In Section~\ref{sec:experiment:scaling}, we analyze two performance-sensitive hyperparameters: the number of non-embedding parameters ($N$) and the number of training tokens ($D$). This section presents a number of performance-invariant designs. The results are shown in Fig.~\ref{fig:apx:ablation}.
\begin{figure}[t]
    \centering
    \includegraphics[width=\linewidth]{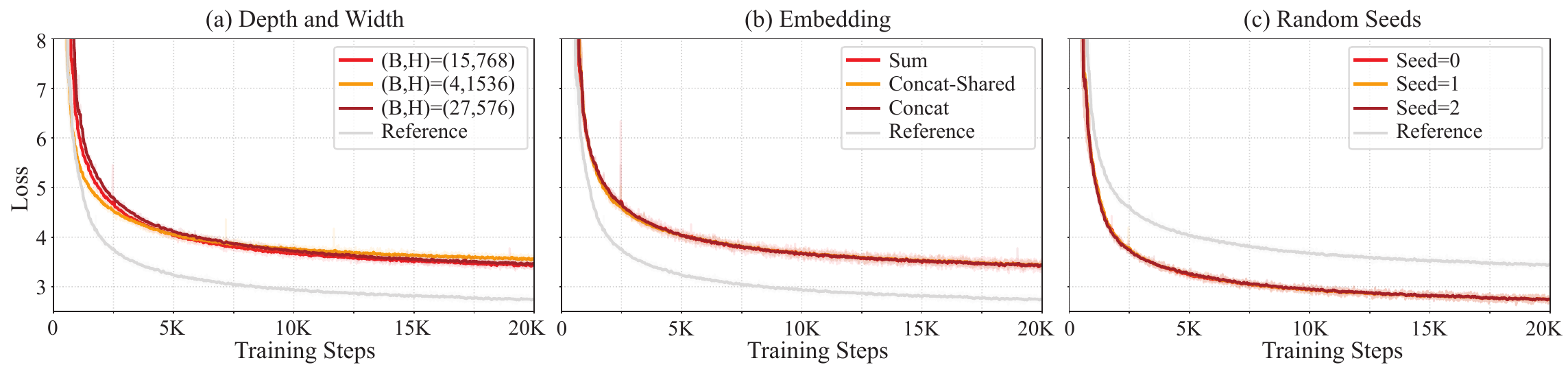}
    \vspace{-1em}
    \caption{Analysis of performance invariance in MDM-Prime and MDM-Prime-v2. The plots show robustness against variations in: (a) model architecture depth and width; (b) embedding table designs; and (c) random seeds for the index shuffling operation.}
    \vspace{-1em}
    \label{fig:apx:ablation}
\end{figure}

\textbf{(a) Depth and Width of the Model.}~As previously established in~\cite{kaplan2020scalinglawsneurallanguage}, the specific shape of a Transformer architecture has a marginal influence on its final performance. Our observations align with their findings. To study this property, we vary the hidden dimension ($H$) and the number of Transformer blocks ($B$) while maintaining a fixed model size. Specifically, we follow the approximation $N \approx 12 B H^2$, increasing the number of blocks quadratically as the hidden dimension is reduced. We compare three distinct $(B,H)$ configurations: $(15,768)$, $(27,576)$, and $(4,1536)$, with their corresponding training curves illustrated in Fig.~\ref{fig:apx:ablation}. For a clear comparison, we adopt the original MDM-Prime ($\ell=16$) and evaluate the impact of varying $B$ and $H$ against a reference curve for MDM-Prime-v2 ($\ell=16$) with $(B,H)=(15,768)$ (shown in gray). We observe that the performance is not sensitive to the model shape.

\textbf{(b) Sub-token Embeddings.}~To adapt an MDM architecture to that of MDM-Prime, input sub-token embeddings are merged into token embeddings prior to the Transformer layers; this prevents an increase in model size~\cite{chao2025mdmprime}. We observe that performance remains invariant across different merging methods. In this experiment, we evaluate three designs: (1) \textit{Concat-Shared}: the original concatenation operation proposed in~\cite{chao2025mdmprime}; (2) \textit{Concat}: a concatenation operation using non-shared embedding tables for sub-tokens at different positions; and (3) \textit{Sum}: a summation operation using the non-shared design from (2). The results are presented in Fig.~\ref{fig:apx:ablation}~(b). Similarly, we include MDM-Prime-v2 ($\ell=16$) with \textit{Concat-Shared} as a reference curve (shown in gray). We observe that the performance is invariant to the choice of the embedding table designs.

\textbf{(c) Random Seed of the Index Shuffling Operation.}~In MDM-Prime-v2, we introduce an index shuffling operation within the subtokenizer $\f_\ell$. To evaluate the robustness of this design to random seed initialization, we independently train three models using different seeds in the random index shuffling operation. The performance of these runs is presented in Fig.~\ref{fig:apx:ablation}~(c), alongside MDM-Prime ($\ell=16$) as a reference baseline (gray line). We observe that the performance is robust to the choice of random seeds.

\subsubsection{Posterior Log-likelihood Distribution Across Time}
\label{sec:apx:experiments:error_distribution}
In Section~\ref{sec:experiment}, we established that MDM-Prime-v2 exhibits improved posterior log-likelihood modeling (i.e., the loss) compared to MDM. To investigate the temporal distribution of this improvement, we computed the log-likelihood at each timestep $t$, specifically $\E_{q_\alpha(\vx_0,\vx_t)}[\log p(\vx_0|\vx_t)]$ for MDM and $\E_{q_\alpha(\vy_0,\vy_t)}[\log p_\ell(\vy_0|\vy_t)]$ for MDM-Prime, omitting the common weighting term $\frac{\alpha'}{1-\alpha_t}$. These results are presented in Fig.~\ref{fig:apx:error_distribution}, where (a) and (b) show the settings without and with Technique 1, respectively.

We observe that MDM-Prime is not consistently superior across all timesteps. There is a distinct crossover point at approximately $t \approx 0.6$ in plot (a) and $t \approx 0.75$ in plot (b). These results indicate that the original MDM performs better when the sequence contains a high ratio of masked tokens (large $t$). On the other hand, MDM-Prime outperforms MDM when sufficient unmasked tokens are present (small $t$). We hypothesize that coarse, token-level prediction is more robust when context is scarce, whereas fine-grained denoising becomes advantageous when sufficient context is available. Motivated by this observation, we introduce a minor modification to the sub-token sampling process: we apply token-wise joint masks for $t=0.5$ in the experiments of Section~\ref{sec:experiment:1B_model}. Let $u$ denote the sub-token group size, which equals to $1$ if $t<0.5$ and $\ell$ if $t\geq0.5$. $\vy^{i,k}_t=[y_t^{i,k*u},\cdots,y_t^{i,(k+1)*u}]$. The modified forward kernel groups sub-tokens is formulated as:
\begin{equation}
\begin{aligned}
q_\alpha(\vy_t|\vy_0)\triangleq \prod_{i=1}^L\prod_{k=1}^{\ell/u} q_\alpha(\vy^{i,k}_t|\vy_0)
\end{aligned}
\end{equation}
This modification incurs no additional FLOPs during training, integrates seamlessly into the existing pipeline, and yields consistent improvements in generation quality.

A second finding is that the index shuffling operation consistently improves log-likelihood across all timesteps, as evidenced by the upward shift in plot (b) for all token granularities $\ell$. This shift moves the crossover point from $t \approx 0.65$ to $t \approx 0.75$. These results demonstrate the effectiveness of the proposed Technique 1.

\begin{table}[t]
\vspace{0.5em}
\begin{minipage}{0.28\linewidth}
\newcommand{\boldtoprule}{\toprule[1.2pt]}
\newcommand{\boldbottomrule}{\bottomrule[1.2pt]}
\renewcommand{\arraystretch}{1.3}
\centering
\begin{center}
    \small
    \begin{tabular}{lccc}
        \boldtoprule
                  & $N$ &   $D$  &  Loss ($\downarrow$) \\
        \hline
        $\ell$=1 &  92M & 52.4B  & 2.925 \\
        $\ell$=4 &  92M & 52.4B  & 3.504 \\
        $\ell$=8 &  92M & 52.4B  & 3.962 \\
        \boldbottomrule
    \end{tabular}
    \vspace{-0.5em}
    \caption{Loss of ARM trained with $\ell$=1, 4, and 8 on OWT. $N$ and $D$ represents the model size and the number of training tokens, respectively. The symbol $\downarrow$ represents that lower values correspond to better performance.}
    \label{tab:apx:arm_l}
\end{center}
\end{minipage}\hfill
\begin{minipage}{0.7\linewidth}
    \centering
    \includegraphics[width=0.975\linewidth]{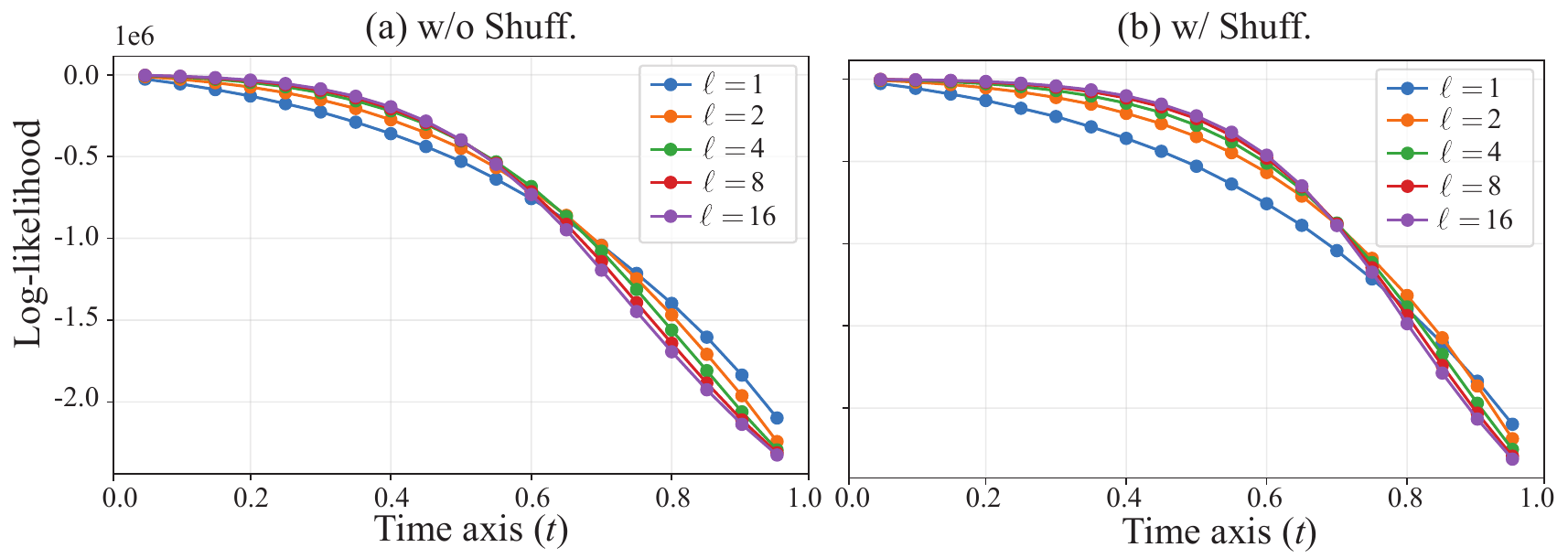}
    \vspace{-0.5em}
    \captionof{figure}{Log-likelihood of MDM ($\ell=$1) and MDM-Prime ($\ell>$1) across time. Subplots (a) and (b) show the results without and with the index shuffling operation, respectively. The maximum value for each curve is 0.}
    \label{fig:apx:error_distribution}
\end{minipage}
\vspace{-0.5em}
\end{table}

\subsubsection{Autoregression on Sub-tokens}
\label{sec:apx:experiments:arm_different_l}
This section presents the results for ARM trained with sub-tokenization using $\ell \in \{1, 4, 8\}$. Let $y_0^k$ denote $y_0^{i,j}$ where $i=\lceil k/\ell \rceil$ and $j=k\%\ell$. The negative log-likelihood of these models is expressed as follows: 
\begin{equation}
\begin{aligned}
\E_{q_\text{data\_y}(\vy_0)}\left[-\sum_{k=1}^{L\ell}\log p(y_0^{k}|\vy_0^{<k})\right],
\end{aligned}
\end{equation}
where $\vy_0^{<k}=[y_0^{1}, \cdots, y_0^{k-1}]$. Each model contains 92M non-embedding parameters and was trained on 52.4B tokens from the OWT dataset~\cite{gokaslan2019owt}. Table~\ref{tab:apx:arm_l} summarizes the results. Contrary to MDM, ARM exhibits performance degradation at higher $\ell$ values. We hypothesize this is because the predictive performance of ARM naturally declines as sequence length increases~\cite{rashkin2020plotmachinesoutlineconditionedgenerationdynamic, pillutla2021mauve}. Furthermore, unlike MDM-Prime, which uses a token-space model architecture (Section~\ref{sec:apx:specification:architecture}), ARM with sub-tokens requires explicit factorization, which increases the total FLOPs by $\ell$ times. For these reasons, we select the token-level ARM as a baseline in Sections~\ref{sec:experiment:scaling} and \ref{sec:experiment:1B_model}.

\begin{figure*}[t!]
    \centering
    \vspace{-0.5em}
    \includegraphics[width=\linewidth]{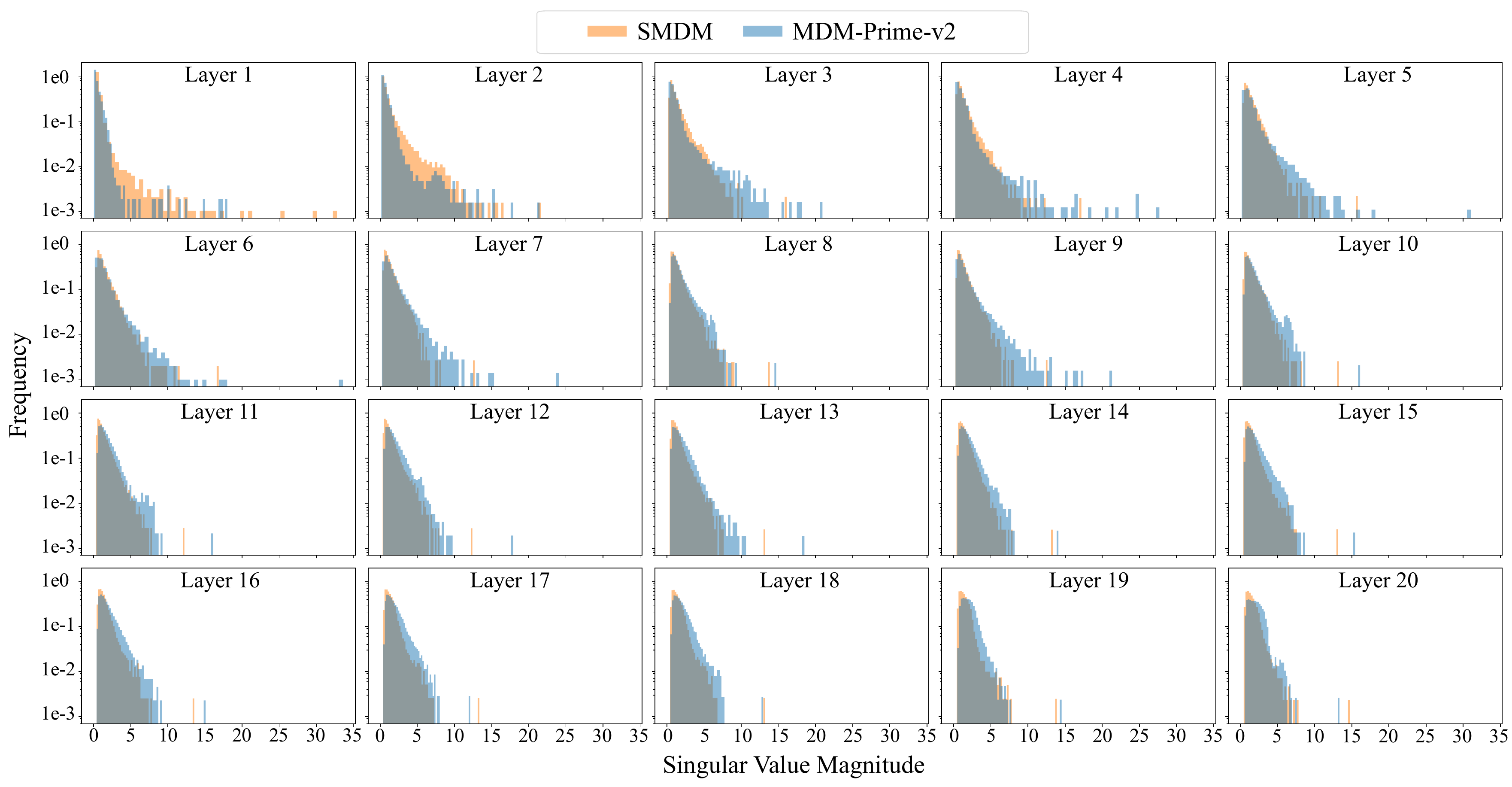}
    \vspace{-1.895em}
    \caption{Singular value distributions of the query-key-value projection matrices of SMDM (1.1B)~\cite{nie2025scalingmaskeddiffusionmodels} and MDM-Prime-v2 (1.1B). Both SMDM and MDM-Prime-v2 adopt the same architecture containing 20 layers. The y-axis represents frequency on a log scale. Compared to MDM, MDM-Prime-v2 exhibits a heavier-tailed distribution, particularly in intermediate layers (i.e., Layers 3–12).}
    \vspace{-0.5em}
    \label{fig:apx:spectra}
\end{figure*}

\begin{figure*}[h!]
    \centering
    \includegraphics[width=\linewidth]{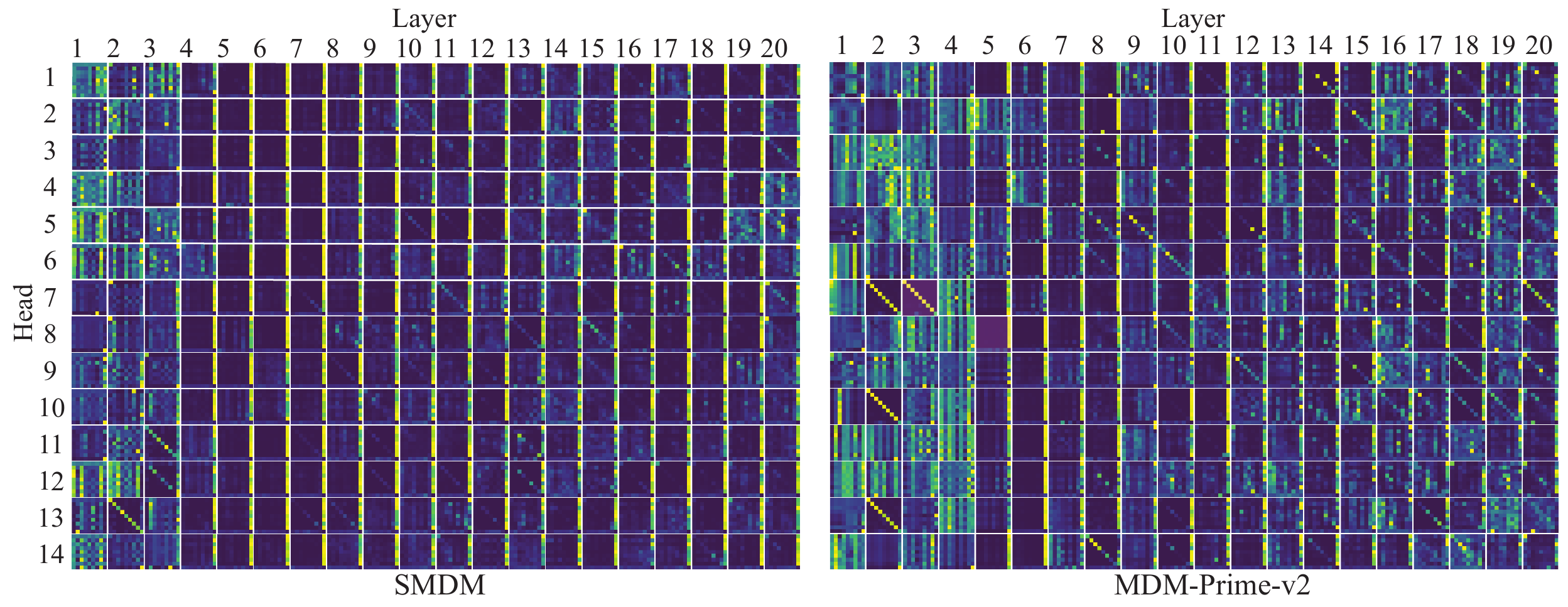}
    \vspace{-1.8em}
    \caption{Attention score patterns of SMDM and MDM-Prime-v2. Both models adopt the same model architecture containing 20 layers and 14 attention heads. SMDM exhibits vertical stripes across deeper layers (e.g., Layers 4–20). In contrast, MDM-Prime-v2 displays diagonal structures and a richer diversity of patterns.}
    \label{fig:apx:attention}
\end{figure*}

\begin{table}[t]
    \renewcommand{\arraystretch}{1.05}
    \newcommand{\boldtoprule}{\toprule[1.2pt]}
    \newcommand{\boldmidrule}{\midrule[0.5pt]}
    \newcommand{\boldbottomrule}{\bottomrule[1.2pt]}
    \centering
    \resizebox{\linewidth}{!}{%
    \begin{tabular}{cc|c|cccccccc|c}
        \boldtoprule
        Technique 1  & Technique 2 & Loss & SciQ & SocialIQA & McTaco & TruthfulQA & BoolQ & ANLI & ARC-e  & OBQA & Avg. \\
        \boldmidrule
                     &              & 2.63 & 68.7  & 34.2  &  53.2 &  26.0 & 48.4 & 34.9  & 33.6 & 21.2  & 40.02 \\
        $\checkmark$ &              & 2.34 & 69.5 & 35.3 & \textbf{65.9} & \textbf{28.4} & 48.8 & 34.4 & 34.0 & 23.8  & 42.51 \\
        $\checkmark$ & $\checkmark$ & \textbf{2.31} & \textbf{73.6} & \textbf{36.2} & 64.9 & 26.6 & \textbf{50.0} &  \textbf{35.4} &  \textbf{39.2} & \textbf{24.2} & \textbf{43.76} \\
        \boldbottomrule
    \end{tabular}}
    \vspace{-0.25em}
    \caption{Ablation study of the proposed two techniques on zero-shot benchmarks. The models contain 170M parameters and are trained on SlimPajama. The original MDM-Prime model is trained with $\ell=3$. MDM-Prime-v2 adopts Technique 1 (random shuffling) and Technique 2 (binary encoding) with $\ell=15$.}
    \vspace{-0.45em}
    \label{tab:experiment:zeroshot_qa_170}
\end{table}

\subsubsection{Ablation Analysis}
\label{sec:apx:experiments:ablation}
This section presents an ablation analysis of the two proposed techniques. We train MDM-Prime (170M parameters) on the SlimPajama dataset and report evaluation results across eight zero-shot benchmarks in Table~\ref{tab:experiment:zeroshot_qa_170}. Both techniques consistently reduce validation loss and improve zero-shot accuracy across all eight commonsense reasoning benchmarks. These results confirm two key findings: (1) lower validation loss correlates with stronger commonsense reasoning performance in MDM-Prime, and (2) the proposed techniques effectively reduce loss and improve accuracy. With the two techniques, the average accuracy increases from 40.02\% to 43.76\% with 3.74\% gains.
\begin{figure}[t]
    \centering
    \vspace{-0.25em}
    \includegraphics[width=0.975\linewidth]{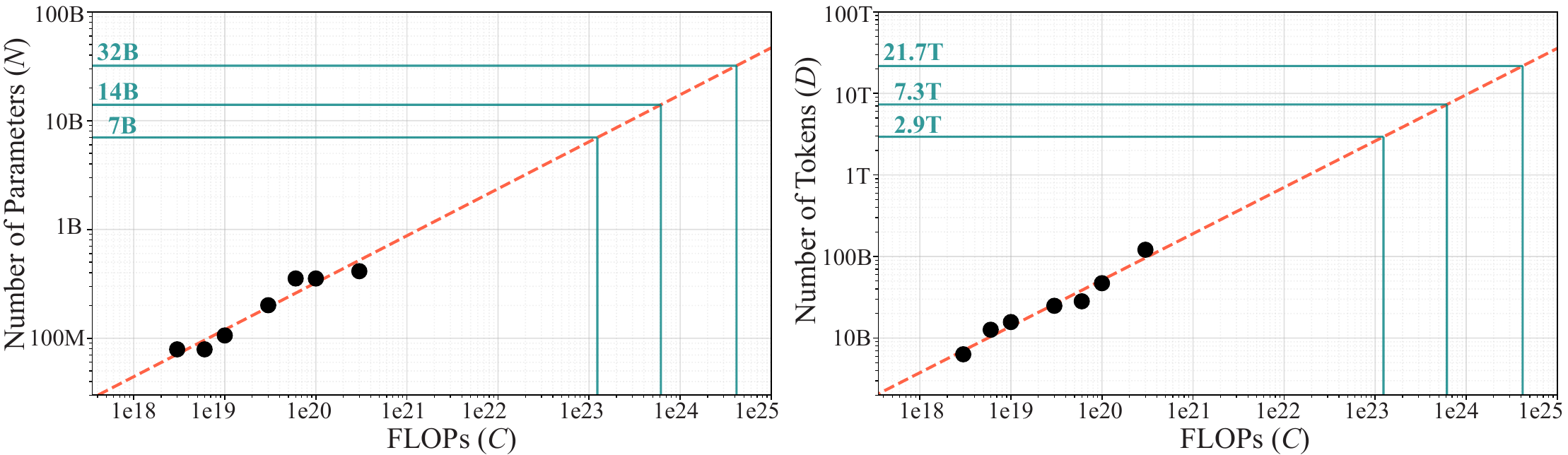}
    \vspace{-0.5em}
    \caption{Compute-optimal allocation of the number of non-embedding parameters ($N$) and training tokens ($D$) for MDM-Prime-v2 under different FLOPs ($C$). Left: Optimal $N$ as a function of $C$. Right: Optimal $D$ as a function of $C$. Black points represent the compute-optimal empirical samples. The red dashed lines indicate the fitted power-law regression extended to larger budgets. Teal annotations highlight the training token counts required to optimally train models of specific target sizes (e.g., 7B, 14B, and 32B).}
    \vspace{-1em}
    \label{fig:apx:scaling_curve_prime}
\end{figure}

\subsubsection{Compute-Optimal Scaling of Parameters and Training Tokens}
\label{sec:apx:experiments:scaling}

To guide future large-scale pretraining research for MDM-Prime-v2, we extend this analysis to determine the compute-optimal allocation of non-embedding parameters ($N$) and training tokens ($D$) for budgets up to $10^{25}$ FLOPs. As illustrated in Fig.~\ref{fig:apx:scaling_curve_prime}, our predictions indicate that for target model sizes of 7B, 14B, and 32B, the corresponding optimal training token counts are 2.9T, 7.3T, and 21.7T, respectively.

\begin{table}[t]
\vspace{1em}
\begin{minipage}{0.38\linewidth}
\newcommand{\boldtoprule}{\toprule[1.2pt]}
\newcommand{\boldbottomrule}{\bottomrule[1.2pt]}
\renewcommand{\arraystretch}{1.3}
\centering
\vspace{-0.25em}
\begin{center}
    \small
    \resizebox{\linewidth}{!}{%
    \begin{tabular}{lccc}
        \boldtoprule
        Tokens     & SciQ &  PIQA  &  ARC-Challenge \\
        \hline
        0.2T &  90.67  &  55.33  & 26.00  \\
        0.5T &  92.67  &  57.33  & 28.67  \\
        1.2T &  92.67  &  59.33  & 30.00 \\
        \boldbottomrule
    \end{tabular}}
    \vspace{-0.5em}
    \caption{Zero-shot evaluation of MDM-Prime-v2 (7B) on SciQ, PIQA, and ARC-Challenge as a function of training tokens. All metrics are accuracy-based; higher values indicate better performance. The results are evaluated with 256 Monte Carlo inference steps and 150 instances on each task.}
    \label{tab:apx:7B}
\end{center}
\end{minipage}\hfill
\begin{minipage}{0.58\linewidth}
    \centering
    \includegraphics[width=\linewidth]{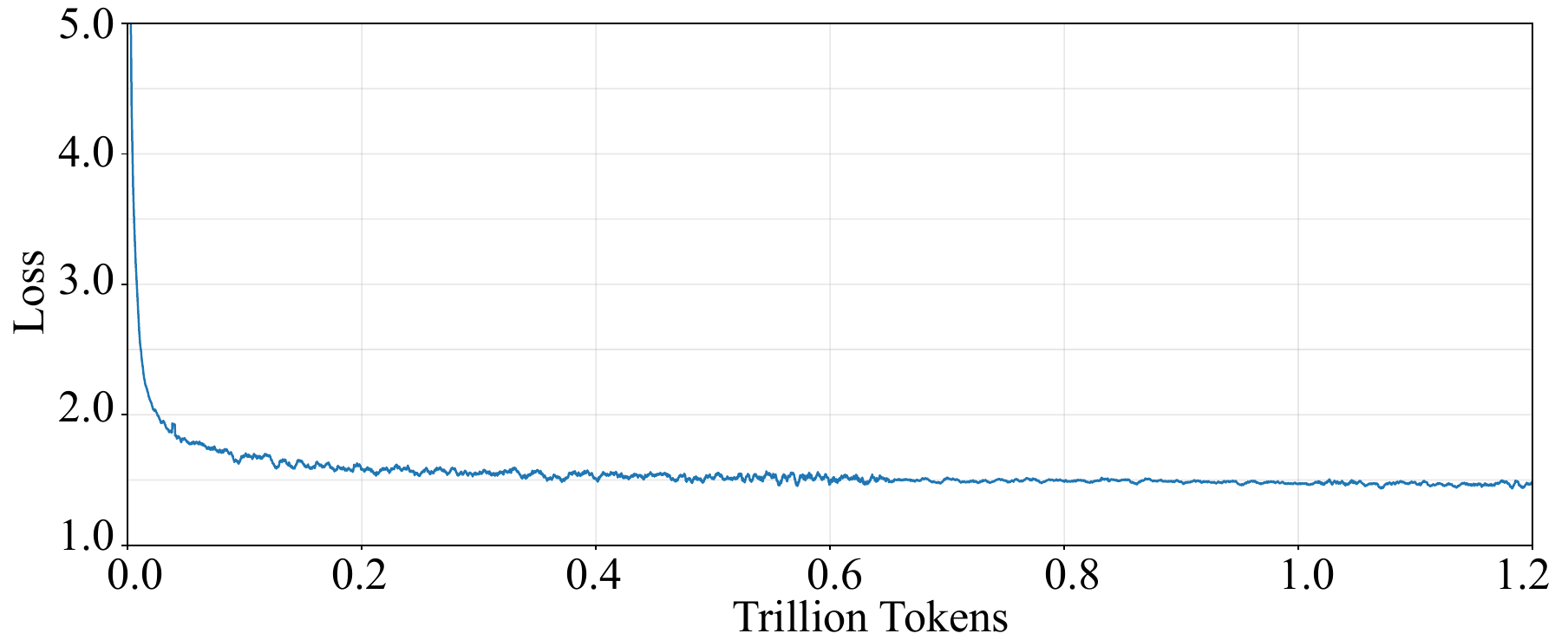}
    \vspace{-1.6em}
    \captionof{figure}{Training loss curve of MDM-Prime-v2 (7B) over the course of pre-training. The model is trained on 1.2T tokens using the Dolma-6T dataset.}
    \label{fig:7B}
\end{minipage}
\vspace{-0.5em}
\end{table}

\subsubsection{Preliminary Results on 7B-Parameter Scale}
\label{sec:apx:experiments:7B}
Current large-scale diffusion language models (e.g., LLaDA 8B~\cite{nie2025llmdiff}, Dream 7B~\cite{ye2025dream7bdiffusionlarge}) are not fully open-sourced and cannot be reproduced. The primary reason is that these models are trained on proprietary datasets, preventing future work from studying their pre-training dynamics. This motivates us to establish a training recipe for fully open diffusion LLMs. We adopt the fully open Dolma-6T dataset~\cite{olmo2025olmo3} and use the standard LLaMA~\cite{touvron2023llama} architecture with a context length of 4,096 for MDM-Prime-v2 (7B parameters). Our preliminary run trains MDM-Prime-v2 (7B) on 1.2T tokens using 192 NVIDIA H200 GPUs. As this run has not yet reached the compute-optimal setting discussed in Section~\ref{sec:apx:configuration:scaling}, where a 7B model should be trained on at least 2.9T tokens, we report only the training curve and downstream performance on SciQ~\cite{Welbl2017CrowdsourcingMC}, PIQA~\cite{Bisk2020}, and ARC-Challenge~\cite{Clark2018ThinkYH} as preliminary results. The loss curve is shown in Fig.~\ref{fig:7B}, and downstream results are presented in Table~\ref{tab:apx:7B}. The most significant challenge we encountered at this scale is training efficiency. For MDMs with bidirectional attention, model FLOPs utilization (MFU) peaks at only 37\%, falling short of the ~45\% typically achieved by standard ARMs. We hypothesize that \texttt{torch.compile} is optimized for causal attention transformers and does not yield full acceleration for diffusion LLMs. Developing more hardware-efficient pre-training methods remains an important direction for future work.



\subsection{Related Works}
\label{sec:apx:related_works}
The foundational framework for discrete diffusion was established by \cite{dickstein2015diffusion} and subsequently expanded to real-world generative tasks via D3PM~\cite{austin2021structurediff} and CTMC~\cite{campbell2022ctmc}. Recent efforts have adapted these principles to large language modeling, where MDMs~\cite{sahoo2024simplifieddiff, shi2024simplifieddiff} have emerged as a particularly scalable framework among various discrete diffusion alternatives~\cite{gulrajani2023plaid, gat2024dfm}. While the scaling behavior of MDMs have been thoroughly investigated under both compute-constrained~\cite{nie2025scalingmaskeddiffusionmodels, ni2025trainingoptimallargediffusion, vonrütte2025scalingbehaviordiscretediffusion} and data-constrained~\cite{prabhudesai2025diffdataconstrained} regimes, a significant efficiency gap against ARMs remains. We investigate the role of the subtokenizer in shaping the MDM-Prime training objective and its influence on scaling behavior and downstream performance. Through a principled analysis of token granularity and sub-token entropy, we develop MDM-Prime-v2 and validate its effectiveness through large-scale pretraining, demonstrating state-of-the-art zero-shot commonsense reasoning performance at the 1.1B parameter scale.

\subsection{Limitations and Discussions}
\label{sec:apx:discussion}

This section outlines three potential directions for further enhancing MDM-Prime-v2.

\begin{table}[t]
\vspace{0.5em}
\begin{minipage}{0.52\linewidth}
\newcommand{\boldtoprule}{\toprule[1.2pt]}
\newcommand{\boldbottomrule}{\bottomrule[1.2pt]}
\renewcommand{\arraystretch}{1.1}
\centering
\vspace{-1.5em}
\begin{center}
    \begin{tabular}{lcc}
        \boldtoprule
                  & Entropy  ($\uparrow$)  &   Loss  ($\downarrow$)   \\
        \hline
         Random shuffling &  0.9936  &  2.860  \\
        Greedy assignment  &  0.9943   &  2.858 \\
        Maximum  & 1.0000 & - \\
        \boldbottomrule
    \end{tabular}
    \vspace{-0.5em}
    \caption{Entropy of sub-tokens across different subtokenization strategies. The results are evaluated on C4. The maximum entropy is given by $-\log_2 \frac{1}{b}$, where $b=\lceil\sqrt[\ell]{V}\rceil$ and $V=50,257$. Random shuffling permutes token indices via a random assignment; greedy assignment iteratively pairs the most and least frequent tokens to balance probability mass across sub-tokens (see CDF in Fig.~\ref{fig:greedy_cdf}).}
    \label{tab:apx:greedy}
\end{center}
\end{minipage}\hfill
\begin{minipage}{0.43\linewidth}
    \centering
    \includegraphics[width=\linewidth]{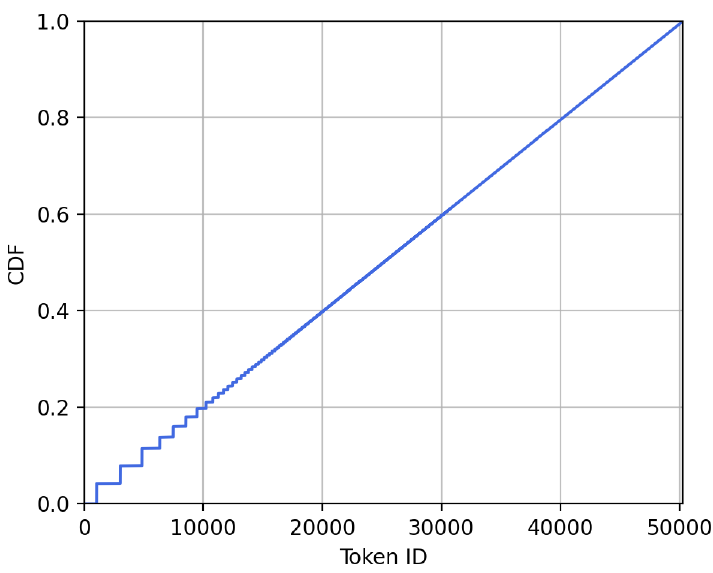}
    \vspace{-2em}
    \captionof{figure}{Cumulative distribution function (CDF) of token indices under greedy assignment, computed using the GPT-2 tokenizer on C4.}
    \label{fig:greedy_cdf}
\end{minipage}
\vspace{-1em}
\end{table}

\paragraph{Optimal Encoding Strategy.} MDM-Prime-v2 adopts a random token index shuffling operation as a simple implementation to address distributional issues caused by the BPE tokenizer. However, \textbf{Proposition~\ref{proposition:entropy_max}} suggests that the optimal design is a function that maximizes the entropy of $\vy_t$. Although our experimental results demonstrate that index shuffling effectively increases entropy—yielding a near-optimal design—we anticipate that a rigorously derived algorithmic approach could further enhance performance. As an ablation, we compare random shuffling against a greedy assignment algorithm that maximizes entropy given corpus token-frequency statistics; results are reported in Table~\ref{tab:apx:greedy} (49M-parameter models trained for $3\times 10^{18}$ FLOPs)  with the corresponding CDF of token IDs under greedy assignment shown in Fig.~\ref{fig:greedy_cdf}. Greedy assignment achieves marginally higher entropy but brings negligible performance improvement ($\Delta$loss = 0.002). Given that it requires additional corpus statistics whereas random shuffling is fully off-the-shelf, we recommend random shuffling for practical use. Nevertheless, designing entropy-maximizing assignment strategies that remain computationally lightweight represents a promising direction for future work within the MDM-Prime-v2 framework.

\paragraph{A Complete Training Pipeline at 7B Scale.} While 7B parameters has become a standard model size among contemporary language models, completing a full compute-optimal pretraining run at this scale remains beyond our current computational budget. A compute-optimal 7B model requires training on at least 2.9T tokens (Section~\ref{sec:apx:configuration:scaling}), yet our preliminary run, spanning approximately two weeks on 192 NVIDIA H200 GPUs including engineering optimization and debugging, reaches only 1.2T tokens. An additional challenge at this scale is training efficiency: MDMs with bidirectional attention achieve a peak model FLOPs utilization (MFU) of only 37\%, compared to approximately 45\% for standard ARMs, likely because \texttt{torch.compile} is optimized for causal attention and does not fully accelerate bidirectional architectures. Completing the compute-optimal 7B run and closing this efficiency gap, through hardware-aware kernel design or compiler support for bidirectional attention, represent important practical milestones for scaling MDM-Prime-v2 to the operating regime of modern language models.

\paragraph{Post-training of MDM-Prime-v2.} This work primarily focuses on improving the pretraining of MDM-Prime via enhanced subtokenizer designs. However, as established in recent works~\cite{nie2025llmdiff}, post-training also plays a critical role in downstream performance. Therefore, establishing an explicit relationship between subtokenizer design and downstream capabilities represents a promising direction for future research.


\subsection{Potential Risks \& AI Assistants}
Our work does not involve direct societal risks, as it is a fundamental research contribution focused on improving the training objective and architecture of masked diffusion language models. The datasets used (C4 and SlimPajama) are publicly available and widely adopted benchmarks. AI assistance was used solely for grammar editing.

\end{document}